\documentclass{article}

\usepackage{microtype}
\usepackage{graphicx}
\usepackage{subfigure}
\usepackage{booktabs} 
\usepackage{lipsum}  
\usepackage{hyperref}
 \usepackage{tcolorbox}
 \usepackage{pifont}
 \usepackage{enumitem}

\usepackage{iclr2025_conference,times}

\usepackage{amsmath,amsfonts,bm}

\def\eqref#1{equation~\ref{#1}}

\def\1{\bm{1}}

\DeclareMathAlphabet{\mathsfit}{\encodingdefault}{\sfdefault}{m}{sl}
\SetMathAlphabet{\mathsfit}{bold}{\encodingdefault}{\sfdefault}{bx}{n}

\usepackage{hyperref}
\usepackage{url}

\iclrfinalcopy

\usepackage{amsmath}
\usepackage{amssymb}
\usepackage{mathtools}
\usepackage{amsthm}
\usepackage{multirow}

\usepackage[capitalize,noabbrev]{cleveref}
\usepackage{mdframed}
\usepackage{wrapfig}

\usepackage{xcolor}         

\usepackage{graphicx}       

\usepackage{dsfont}         

\PassOptionsToPackage{comma}{natbib}

\newcommand{\yunzhen}[1]{\textcolor{red}{Yunzhen: {#1}}}

\newcommand{\remove}[1]{}

\newcommand\blankfootnote[1]{%
  \let\svthefootnote\thefootnote%
  \let\thefootnote\relax%
  \footnotetext{#1}%
  \let\thefootnote\svthefootnote%
}

\theoremstyle{plain}
\newtheorem{theorem}{Theorem}[section]
\newtheorem{proposition}{Proposition}[section]

\newtheorem{corollary}[theorem]{Corollary}

\newtheorem{assumption}[theorem]{Assumption}
\newtheorem{remark}[theorem]{Remark}
\newtheorem{condition}[theorem]{Condition}

\usepackage{caption}

\usepackage{subcaption}
\usepackage{amsmath}
\usepackage{bm} 
\usepackage{amsthm} 
\usepackage{wrapfig}

\title{Beyond Model Collapse: Scaling Up with Synthesized Data Requires Verification}

\author{Yunzhen Feng\footnotemark[1] \  \footnotemark[2] \ \footnotemark[4] \ \footnotemark[6] \quad Elvis Dohmatob\footnotemark[1] \ \footnotemark[4] \quad Pu Yang\footnotemark[1] \ \footnotemark[5] \quad Francois Charton\footnotemark[4] \quad Julia Kempe\footnotemark[2] \ \footnotemark[3] \ 
 \footnotemark[4]  \\
\footnotemark[1] \ Equal Contributions \\
\footnotemark[4] \ Meta FAIR \\
   \footnotemark[2] \ Center for Data Science, New York University \\
   \footnotemark[3] \ Courant Institue of Mathematical Sciences, New York University \\
   \footnotemark[5] \ School of Mathematical Sciences, Peking University \\
   \footnotemark[6] \ Part of the work was done during an internship at Meta. \\
   \texttt{yf2231@nyu.edu} \\
}

\begin{document}

 \maketitle

\begin{abstract} 

Large Language Models (LLM) are increasingly trained on data generated by other LLM, either because generated text and images become part of the pre-training corpus, or because synthetized data is used as a replacement for expensive human-annotation. This raises concerns about \emph{model collapse}, a drop in model performance when their training sets include generated data. Considering that it is easier for both humans and machines to tell between good and bad examples than to generate high-quality samples, we investigate the use of verification on synthesized data to prevent model collapse. We provide a theoretical characterization using Gaussian mixtures, linear classifiers, and linear verifiers to  derive conditions with measurable proxies to assess whether the verifier can effectively select synthesized data that leads to optimal performance. We experiment with two practical tasks -- computing matrix eigenvalues with transformers and news summarization with LLMs -- which both exhibit model collapse when trained on generated data, and show that verifiers, even imperfect ones, can indeed be harnessed to prevent model collapse and that our proposed proxy measure strongly correlates with performance. 



\end{abstract}

\section{Introduction}
\label{sec:intr}



As generative models for language, images, and video continue to achieve human-level performance \citep{touvron2023llama, achiam2023gpt, pmlr-v139-ramesh21a, Rombach_2022_CVPR, Sora}, they are increasingly used to synthesize data across diverse domains, including coding \citep{haluptzok2022language} and mathematics \citep{trinh2024solving}. With this rise in AI-generated data, there is growing interest in its potential to replace expensive human annotators.

This gradual replacement of human-written corpora by machine-generated tokens gives rise to a number of concerns, notably the risk of ``model collapse''~\cite{shumailov2023curse,Shumailov2024Nature}, where iterated training on synthesized data brings a drop in model performance, and, ultimately, ``dumber models''. This phenomenon was observed empirically~\citep{Hataya_2023_ICCV, martínez2023combining, martínez2023understanding, bohacek2023nepotistically, briesch2023large, guo2023curious} and described theoretically~\citep{alemohammad2023selfconsuming, bertrand2023stability, dohmatob2024model}. Its main consequence is the breaking of known scaling laws~\citep{dohmatob2024tale}: as data becomes more  synthetic, larger training sets do not enhance performance. Current thinking around model collapse recommends that synthetic data is to be avoided for model training lest the system deteriorate completely.

 This raises a critical question: are synthesized data so fundamentally flawed that they must be identified and discarded using detection techniques \citep{zhang2024regurgitative}, corrected through various refinement methods \citep{gillmanself} or used as negative feedback \citep{alemohammad2024selfimprovingdiffusionmodelssynthetic}?
In this work, we argue that synthesized data contain abundant valuable information that can be effectively leveraged. 
We advocate for a paradigm of verifier-based selection of synthesized data and provide criteria 
for when this process is successful - even with imperfect verifiers. 

Throughout this paper, we focus on a scenario where synthesized data serve as the labels (answers) to questions—a common framework in tasks such as question answering, code generation, and mathematical reasoning.
 To begin, we explore the information contained in synthesized data through a controlled experiment: training transformers on a linear algebra task. This setting allows us to both understand the solutions and easily synthesize data, placing it into a long established corpus of works that study interesting phenomena in a controlled setting to advance our understanding of the underlying mechanisms in larger models in the wild (see e.g. \cite{power2022grokking, garg2022can, charton2023transformers,dohmatob2024tale}).
In this context, we find that the accuracy of the most accurate synthesized solution among the top candidates is three times higher than that of the solution selected by the model itself based on perplexity — its own measure of quality. 
This finding provides a positive answer to our question: the model is indeed capable of generating high-quality solutions. However, it cannot intrinsically {\em identify} the best solution using perplexity alone. Thus, effectively scaling with synthesized data requires a robust selection and verification process.

We proceed to provide a theoretical characterization using Gaussian mixtures and linear classifiers, employing an external verifier to select the generated data. In the limit of infinite synthesized data, we analyze the conditions required for the generator and verifier to enable the model, trained on selected synthesized  data, to achieve Bayes-optimal results. Notably, there is a sharp phase transition — from zero accuracy due to errors in the synthesized data to optimal accuracy. Our theory confirms that there is abundant information within the synthesized data, and that proper, not necessarily perfect, selection can overcome model collapse. We also identify a measurable proxy function that 
characterizes the ability of the verifier to select synthesized data that results in good models when retraining on it. 

We conduct two large-scale experiments 
to test our theoretical insights 
: (1) the above mentioned transformers on a linear algebra task and (2) news summarization using the LLM Llama-2 \citep{touvron2023llama}. In both experiments, we apply various verifiers for selection and compute {our theoretical proxy function}. Our results show that relying solely on generated data leads to poorer performance than that of the generator itself, even with increased data volume, indicating model collapse. However, even with noisy and imperfect selection methods, synthesized data can be improved to yield models that outperform the original generator. Our proxy function strongly correlates with final performance across all cases. In particular it elucidates the counter-intuitive fact that a stronger model is not automatically a better selector — as we demonstrate by showing inferior performance when selecting with Llama-3 compared to self-selection with Llama-2 for Llama-2-generated text.

\begin{figure}[tb]
\vspace{-25pt}
    \centering
    \subfigure[The generator is capable of producing high-quality data, and verifiers can extract this data through selection to prevent model collapse, which occurs in the absence of such selection.]{
    \includegraphics[width=0.45\linewidth]{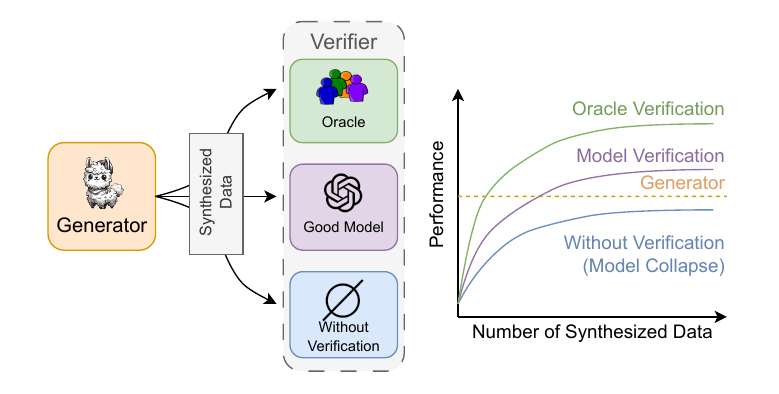}}
    \hspace{15pt}
    \subfigure[In theory, we consider a Gaussian mixture model with a linear generator and linear pruner. The pruner improves synthesized data through selection, resulting in improved performance.]{
    \includegraphics[width=0.45\linewidth]{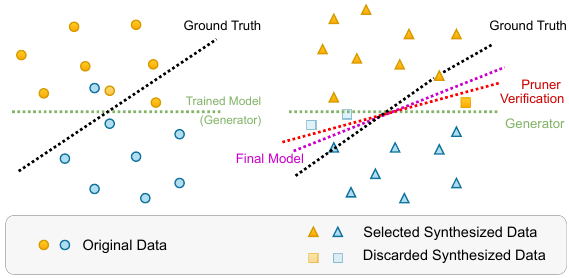}}
    \caption{\small{Illustrative figures for our proposal \textbf{(a)} and for the theoretical and simulation settings \textbf{(b)}.}}
    \label{fig:teaser}
    \vspace{-15pt}
\end{figure}

We summarize our contributions as follows:
\begin{itemize}[noitemsep, topsep=0pt, leftmargin=20pt]
\item We provide a theoretical analysis in the high-dimensional limit with infinite data proving that synthesized data with appropriate selection can lead to optimal performance, and provide a characterization of necessary conditions via a measurable proxy function (Section \ref{sec:theory_main_paper}). 
\item We validate our theoretical findings in three empirical settings of increasing departure from our theoretical assumptions:
\begin{itemize}[noitemsep, topsep=0pt, leftmargin=20pt]
\item Simulation results with linear classifiers, tested in a finite-data regime (Section \ref{sec:toy}),
\item Linear algebra tasks with transformers trained in a generative way
(Section \ref{subsec:exp-math}),
\item News summarization using Llama-2
(Section \ref{sec:sub-news}). 
\end{itemize}

\end{itemize}

Crucially, note that all that is needed to re-attain model performance akin to training on clean original data is the ability to {\em rank} high-quality from low quality labels; arguably, a task much simpler than annotating the labels. Thus, to go beyond model collapse and continue scaling up with synthesized data, {\em verification is all you need!}

\vspace{-5pt}
\section{Related Work}
\label{sec:relat}

We here only list works of direct relevance to ours, with an extensive reference list in Appendix \ref{sec:related_append}.
%
%
%
\paragraph{Model Collapse.} With the advancement of generative models, synthesized data generated by these models has become increasingly prevalent online, mixing irreversibly into our training corpora. Recent studies have highlighted the potential for dramatic deterioration in downstream models, a phenomenon known as {\em ``model collapse"} \citep{shumailov2023curse}.  Empirical studies have demonstrated this issue in various settings \citep{Hataya_2023_ICCV, martínez2023combining, martínez2023understanding, bohacek2023nepotistically, briesch2023large}. Synthesized datasets have been shown to reduce diversity \citep{padmakumar2024writing, guo2023curious} and cause distributional distortions \citep{lebrun2021evaluating}. Theoretical analysis also examines the effects of iterative training on self-generated data \citep{alemohammad2023selfconsuming, bertrand2023stability, dohmatob2024model, seddik2024bad}. Notably, \cite{dohmatob2024tale} warns that model collapse signifies a break in customary neural scaling laws \citep{kaplan2020scaling,hoffmann2022trainingChinchilla}, where increasing synthesized data volume does not enhance performance as effectively as scaling with human-generated data. As a result, recent works have focused on avoiding or correcting synthetic data to prevent model collapse. \cite{gillmanself} propose using a correction function informed by expert knowledge to modify the synthesized data. \cite{alemohammad2024selfimprovingdiffusionmodelssynthetic} leverage a model trained on synthetic data as negative guidance for diffusion models. \cite{zhang2024regurgitative} employ the confidence score and an AI detection classifier to discard synthesized data. In contrast, we propose leveraging the synthesized data through selection techniques.

\textbf{Benefits of Synthesized Data.} Synthetic data holds great potential, as it is much easier and cheaper to scale compared to human-labeled data.  Numerous empirical studies have demonstrated the benefits of synthesized data across a wide range of settings. Common practices include cases where the downstream task slightly differs from that of the data-generating model, where the generating model is significantly stronger than the consuming one, or when better prompt engineering and external information are utilized. In Appendix \ref{sec:taxonomy}, we provide a taxonomy that outlines when and how synthesized data can be advantageous. Data selection is already employed in some works, particularly in code generation and mathematics, where natural verifiers such as compilers, solutions, or heuristic verifiers exist. For instance, \cite{haluptzok2022language} generate synthesized code and filter out incorrect samples. \cite{ulmer2024bootstrapping} use conversational metrics to filter synthetic dialogue data. \cite{trinh2024solving} utilize a symbolic deduction engine to verify correct solutions for Olympiad geometry problems. \cite{setlur2024rl} apply a final answer verifier to distinguish between good and bad synthetic data. Although verifiers are used in these cases, their effects on performance have not been systematically explored, especially in terms of how different types of verifiers influence outcomes.

\textbf{Data Selection.} Data selection is a preprocessing technique typically applied to original datasets with human labels, as discussed in various related works (see Appendix \ref{sec:related_append_3}). In contrast, our work focuses on data selection for {\em synthesized} data, which does not represent a noisy version of the ground truth but rather a skewed distribution. Our aim is not to propose new selection methods but to demonstrate how selection can prevent model collapse and to analyze the conditions under which it is effective.


\section{Warmup}\label{sec:warmup}

We first focus on transformer models on a mathematical task, which offers an interpretable setting to understand the generation quality since we have a clear metric for error and a computable ground truth, as well as the ability to generate as much (original)  data as desirable. 

\label{subsec:math_motivation}

\textbf{Setting.} We follow \citet{charton2022linear}, who showed that transformers \citep{transformer17} can learn to predict the eigenvalues of $5\times 5$ real symmetric matrices with independent and identically distributed entries, rounded to three significant digits. All training, test, and synthesized data are generated by sampling matrices with independent entries from $\mathcal{U}[-10, 10]$. A prediction is considered correct if the relative error in the $L^1$ norm is below a certain tolerance $\tau$. The synthesized data generator is trained on a limited sample of 200,000 examples with Adam for 65 epochs. Details on the tokenizer and optimization can be found in Appendix \ref{sec:append-math}.

We aim to evaluate the quality of the generator's outputs and determine how much useful information can be extracted from it. To achieve this, we employ \textit{beam search}, an inference technique that maintains multiple candidate sequences (referred to as "beams") during the generation process to explore a larger portion of the solution space. Beam search tracks the top $k$ candidate sequences based on their cumulative probability scores, expanding each by considering all possible next tokens. When $k=1$, beam search is equivalent to greedy decoding. The algorithm continues this expansion until the full generation process is complete, allowing us to assess the upper bound of the generator's capabilities and better understand its limitations and strengths.

\begin{wraptable}{r}{0.5\textwidth} 
\vspace{-11pt}
\small
    \centering
    \caption{\small \textbf{Generator accuracy for different beam sizes.} Left: the most accurate solution in the beam is evaluated.
    Right: the solution with lowest perplexity is evaluated.} 
    \vspace{-0.2cm}
    \setlength{\tabcolsep}{3.6pt}
    \begin{tabular}{l|ccc||ccc}
    \toprule 
    & \multicolumn{3}{c|}{Verify} & \multicolumn{3}{|c}{Verify} \\
    & \multicolumn{3}{c|}{all beams} & \multicolumn{3}{|c}{the best beam} \\
    \midrule
    Tolerance $\tau$  & 2\% & 1\% & 0.5\% & 2\% & 1\% & 0.5\% \\
    \midrule
     Beam 50  & \textbf{90.4} & \textbf{60.4} & \textbf{22.9} & \textbf{65.9} & \textbf{19.2} & \textbf{2.4} \\
     Beam 35 & 89.2 & 56.9 & 19.8 & 66.0 & 19.2 & 2.4 \\
     Beam 25 & 88.0 & 53.2 & 16.8  & 66.1 & 19.3 & 2.4 \\
     Beam 10 & 83.7 & 43.1 & 10.5 & 66.2 & 19.5 & 2.5 \\
     Beam 5 & 79.3 & 34.9 & 7.1 & 66.5 & 19.7 & 2.4 \\
     Greedy & \textbf{66.9} & \textbf{20.2} & \textbf{2.4} & \textbf{66.9} & \textbf{20.2} & \textbf{2.4} \\
     \bottomrule
    \end{tabular}
    \label{tab:modelAbeam}
\vspace{-20pt}
\end{wraptable}

In Table \ref{tab:modelAbeam}, we report the test accuracy of generated predictions using beam search. We consider two settings: evaluating only the best beam identified by the model (right) to understand the model's actual predictions, and another evaluating all $k$ candidates in the beams, with the best one contributing to the accuracy (left) to assess the model's full potential.

For self-selection, increasing the number of beams does not lead to any improvement in accuracy, despite significantly higher inference costs. However, in the left setting, moving from greedy decoding (beam 1) to beam 50 improves accuracy from 20.2\% to 60.4\%, with 
$\tau=1\%$. This indicates that \textbf{the model has considerable potential to generate improved solutions, but it lacks the inherent capability to autonomously select the best predictions}. Among the vast amount of synthesized data, there may be high-quality data, but effective selection is necessary.

\vspace{-5pt}

\section{Theoretical Insights}
\label{sec:theory_main_paper}

In this section, we aim to theoretically explore training and verification using synthesized data, and characterize the conditions under which verification enhances performance. To streamline our analysis, we focus exclusively on training with synthesized data, since the inclusion of additional real data is always advantageous. We consider learning with a family of high-dimensional data distributions. Within this context, the verification process is implemented as a {\em pruning strategy} applied to synthesized data. Crucially, we will not assume that the pruning strategy has access to the ground truth, as such an assumption would be overly restrictive for practical applications; rather, we will formulate our theory for general pruners, {which could for instance be} like another trained model. A full exposition of our general theory is in Appendix \ref{app:theory}
; for ease of exposition, we specialize here to Gaussian mixtures.

\subsection{Setting} \label{subsec:3.1}
\paragraph{Data Distribution.} We will consider distributions $P$ over $\mathbb R^d \times \{0,1\}$ with certain high dimensional concentration properties of a general form (Condition \ref{cond:main}). A special case are  binary {\em Gaussian mixtures}:
 features have
  conditional distribution given by
$
x \mid y \sim N(\mu_y,\Sigma),
$
where $\mu_y = (2y-1)\mu$, for some $\mu \in \mathbb R^d$ and $\Sigma$ is a positive-definite matrix  with $\mathbb E\,\|x\|^2 = \|\mu\|_2^2+\operatorname{tr}\Sigma = 1$. For further ease of exposition we will only consider balanced distributions: $$\mathbb P(y=1) = \mathbb P(y=0) = 1/2,\text{ for }(x,y) \sim P.$$ 


\paragraph{Synthesized Data.} Let $D_N=\{(x_1,y_1),\ldots,(x_N,y_N)\}$ be a dataset of $N$ i.i.d. pairs from the true distribution $P$ and let $D_N'=\{(x_1,y_1'),\ldots,(x_N,y_N')\}$ be the synthesized data generated from the same distribution, but where label $y'_i$ (instead of $y_i$) has been generated by an AI model. 


\textbf{Downstream Model and Pruning.} We will model our data selection (whether with or without feedback) via a {\em pruning strategy} $q=(q_1,\ldots , q_N)$ where $q_i$ is a bit which indicates whether the $i$th training example from $D_N'$ has survived pruning. For the downstream models we consider the family:
$$
\mathbb P (y =1\mid x,w) = \hat y := \sigma(x^\top w) \in (0,1), \quad \quad \sigma(z):=\frac{1}{1+e^{-z}}
$$
parametrized by a vector of weights $w \in \mathbb R^d$ and  sigmoid non-linearity $\sigma$. 
Let $\widehat w_N$ be obtained via logistic regression fitted on $D'_N$ with ridge regularization parameter $\lambda>0$. Thus, $\hat w$ is the  unique 
minimizer of the following objective function:
\begin{eqnarray}
L(w) := \frac{1}{N}\sum_{i=1}^N q_i\ell(\sigma(x_i^\top w), y_i') + \frac{\lambda}{2}\|w\|^2,
\label{eq:objective}
\end{eqnarray}
 where $\ell$ is the binary cross-entropy. The corresponding downstream classifier is $\widehat f_N = f_{\widehat w_N}$, where the notation $f_w$ refers to the linear classifier induced by a weights vector $w \in \mathbb R^d$, i.e $f_w(x) = (\operatorname{sign}(x^\top w) + 1) /2$.
\paragraph{Test Accuracy.} The test accuracy of the downstream model $\widehat f_N$ is defined by
$$
acc(\widehat f_N) := \mathbb P(\widehat f_N(x) = f_{\text{Bayes}}(x)),\text{ for a random test point }(x,y) \sim P,
$$
where $f_{\text{Bayes}}(z) := \mathbb E[y |x=z]$ is the Bayes-optimal classifier. In particular, note that $acc(f_{\text{Bayes}})=100\%$ by construction.
The quantity $acc(\widehat f_N)$  will be the main object of our analysis, and we will be interested in how it depends on errors in the generator $P$ and the choice of pruning strategy $q$, in the infinite-sample limit $N \to \infty$.


\subsection{Pruning Strategy}\label{subsec:3.2}

We consider a wide class  of parametrized pruning strategies $q$, which we term {\em Verification-Pruning} (see Appendix \ref{app:theory}).
They satisfy the following reasonable property:
\begin{assumption}[Independent Selection]
\label{ass:independent-selection}
The bits $q_1,\ldots,q_N \in \{0,1\}$ are independent. Thus, in particular, whether any training example $e_i:=(x_i,y_i') \in D_N'$ survives pruning or not is independent of what happens to the other examples $e_{j \ne i}$. 
\end{assumption}
We shall denote by $p \in [0,1)$, the probability that the label $y'$ of a synthesized example $(x,y')$ is different from the true label $y$, i.e
\begin{eqnarray}
    p := \mathbb P(y'\ne y).
\end{eqnarray}
Note that $p$ does not depend on the example index $i$, due to the i.i.d. assumption.

Our verification-pruning family (refer to Appendix \ref{app:theory} for details) is described by four parameters $(\phi_0,\phi_1,\psi_{01},\psi_{10})$, defined as follows:
\begin{eqnarray}\label{equ:psiphi}
\phi_k = \mathbb P(q = 1 \mid y'=k,y=k),\quad \psi_{k\ell}= \mathbb P(q = 1 \mid y'=\ell,y=k).
\end{eqnarray}

For simplicity of exposition, we will focus on {\em symmetric} pruning strategies,
$
\phi_1 = \phi_0 = \phi$ and $\psi_{01}=\psi_{10} = \psi.
$
Assumption \ref{ass:independent-selection} implies that for any class labels $k,\ell \in \{0,1\}$, the random variables $(z_{ik\ell})_{i \in [N]}$ defined by $z_{ik\ell}=1[y_i=k,y_i'=\ell,q_i=1]$ are i.i.d.  with Bernoulli($p_{k\ell}$) distribution, with
\begin{eqnarray}
\label{eq:pkl_dummy}
    p_{kk} = (1-p)\phi_k/2,\text{ and }p_{k\ell}=p\psi_k/2\text{ if }k \ne \ell.
\end{eqnarray}

 In this section we focus on a special case of {\em supervised} pruning strategies $q$ of the form:
\begin{eqnarray}\label{eq:pruner}
\text{\bf Supervised Pruning:} \quad q_i = 1[y'_i(x_i^\top w_{prune})>0],
\end{eqnarray}
for some $w_{prune} \in \mathbb R^d$.  This pruning strategy filters out all examples on which there is disagreement on the assigned label. In Appendix \ref{app:theory} we show how we can map this to $(\phi,\psi)$-pruning.


Let us provide two more notable examples of (symmetric) pruning strategies.

\textbf{No Pruning.} The case $(\phi,\psi) = (1,1)$ corresponds to no pruning, i.e using the entire training dataset. 

\textbf{Oracle Pruning.} The case $(\phi,\psi) = (1,0)$. The pruning strategy only keeps indices corresponding to examples in the dataset which have correct label (all corrupted labels are discarded).

\subsection{Performance of Models Trained with Pruning: Insights from Infinite-Sample Regime}    
The following is our main theoretical result (see Theorem \ref{thm:main} for full statement). It characterizes test accuracy $acc(\widehat f_N)$ of the downstream model on pruned data as a function of $p$ (the label disagreement) and the parameters $(\phi,\psi)$ of the pruner, in the theoretical limit of infinite training data ($N \to \infty$).
\begin{theorem}[Simplified version of Theorem \ref{thm:main}]
\label{thm:main_dumbeddown} 
Let Assumption \ref{ass:independent-selection} be in order.
Fix $p$, $\phi,\psi$ and define the breakdown point $p_\star \in (0,1)$ by $p_\star := 1/(1+\psi/\phi)$. For the family of data distributions obeying Condition \ref{cond:main} (including the Gaussian mixture), for a downstream model $\widehat f_N$ trained on data from  a generator with error rate $p$, pruned with an verification-type strategy  with parameters $(\phi,\psi)$, in the limit $N \to \infty$ it holds a.s that:

(i) 
If $p<p_\star$ then $acc(\widehat f_N) = 100\%$.

(ii) 
If $p>p_\star$ then  $acc(\widehat f_N) = 0\%$. The pruner is overwhelmed by so many inaccuracies in the synthesized data, and the downstream model learns the exact opposite of the true class labels.
\end{theorem}

\begin{wrapfigure}{r}{0.455\textwidth}
    \centering
    \vspace{-17pt}
        \includegraphics[width=0.8\linewidth]{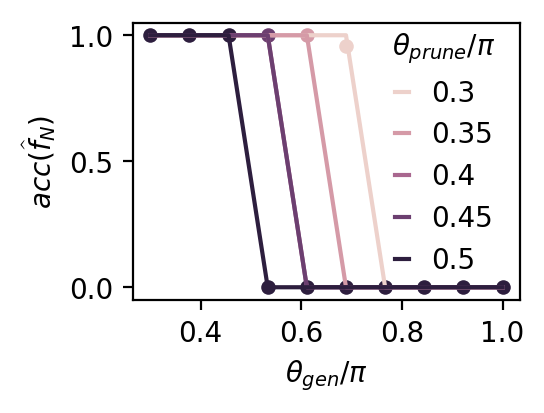}
        \vspace{-10pt}
   \vspace{-0.2cm}
    \caption{\small{\textbf{Empirical confirmation of Theorem \ref{thm:main_dumbeddown}.} Comparing the breakdown points of different generators and pruners of different strengths.
    Synthesized data is generated from a linear model $w_{gen}$ with classification error rate $p=\theta_{gen}/\pi \in [0,1]$. 
    The data is pruned with another linear model $w_{prune}$ which has classification error $\theta_{prune}/\pi$.
    Broken lines correspond to the prediction of Theorem \ref{thm:main_dumbeddown}, while solid points correspond to experiments.
    Notice the sharp phase transitions where the model suddenly switches from perfect accuracy to worse-than-chance, as the theorem predicts. }}
    \label{fig:curves_main}
    \vspace{-45pt}
\end{wrapfigure}

Thus, there is a sharp phase-transition around the corruption level $p_\star := 1/(1+\psi/\phi)$: as $p$ is increased past level $p_\star$, the downstream model $\widehat f_N$ abruptly switches from being perfectly accurate, to perfectly inaccurate! The proof (see Appendix \ref{app:sketch} for a sketch) explicitly computes empirical test accuracy in terms of $N_{k\ell} := \sum_{i=1}^N z_{ik\ell}$, which follow a binomial distribution, bounding the gap to the population accuracy, and using concentration of measure type techniques.
Note that the sharp transition is due to the infinite-sample regime, where we can avoid finite-sample corrections. $p_*$ could be used as a measurable proxy in the experiment.

See Figure \ref{fig:curves_main} (and Figure \ref{fig:curves} in Appendix \ref{app:theory}) for an empirical illustration of the theorem.

\begin{remark}
Note that the $100\%$ accuracy achievable in Theorem \ref{thm:main_dumbeddown} is idealized, and is expected to only hold in the infinite sample regime (with a possibly large but fixed input dimension).
\end{remark}

\subsection{Some Consequences of Theorem \ref{thm:main_dumbeddown}}\label{sec:thm-interpret}
We now present some illustrious applications of Theorem \ref{thm:main}. These examples are empirically confirmed in Figure \ref{fig:curves_main} (see also Figure \ref{fig:curves} in Appendix \ref{app:theory}).

\paragraph{No Pruning.}
Here, we have $\psi/\phi=1$ and so the downstream model achieves $100\%$ accuracy for all values of corruption parameter $p$ up to the breakdown point $p_\star = 1/2$ predicted by Theorem \ref{thm:main} .

\paragraph{Oracle Pruning.} For this scenario, $\psi/\phi=0$ and  so Theorem \ref{thm:main} predicts that the downstream model $\widehat f_N$ achieves $100\%$ accuracy for all values of corruption parameter $p$ up to the breakdown point $p_\star = 1$. This is perhaps not so surprising in hindsight. The point is that even for moderately large values of $\psi/\phi$, the breakdown point $p_\star$ can still be quite close to $1$.

\paragraph{Supervised Pruning.}
Consider isotropic Gaussian mixture data with means $\pm \mu$, and a pruning strategy as in Eq.~(\ref{eq:pruner}). The parameters $(\phi,\psi)$ only depend on the angles $\theta_{gen},\theta_{prune},\theta \in [0,\pi]$: 
\begin{eqnarray}
\label{eq:principal-angles}
\begin{split}
\theta_{gen}&:=\angle (w_{gen},\mu),\, \theta_{prune} := \angle (w_{prune},\mu),\\
\theta &:= \angle (w_{prune},w_{gen}).
\end{split}
\end{eqnarray}
This is because, the $p_{k\ell}$'s defined in \eqref{eq:pkl_dummy} now correspond to \emph{orthant probabilities} for a trivariate normal distribution, with correlation coefficients given by these angles (see also Figure \ref{fig:curves_main}). In the Appendix \ref{app:consequences}, we provide the calculation from the angles to $\psi$ and $\phi$. In practice, we can directly measure the proxy $p_*$, which encapsulates all the aforementioned correlations.

\paragraph{Decoupling the Generator and Verifier.} Although the generator and verifier are coupled together in supervised pruning, there are some intuitions that help us decouple them: (1) a better generator always improves performance, (2) when the verifier is poor, such as in cases of no pruning or random pruning, we have a low breakdown point and require a good generator to achieve good performance, and (3) when the verifier is sufficiently good, close to an oracle, the breakdown point is high, and any non-degenerate generator is sufficient.

\vspace{-5pt}
\section{Simulations on Synthesized Data}\label{sec:toy}


The theoretical results are based on the best-case scenario of having unlimited access to synthesized data and operating in the high-dimensional limit. In this framework, the impact of generator and verifier on performance is reflected in binary outcomes: 100\% or 0\%. In this section, we present simulation results from finite regimes to illustrate the practical implications of our theoretical framework. Specifically, we explore how the generator and verifier influence performance, and how performance scales with an increasing, yet finite, volume of synthesized data.

\subsection{Setting}


Following the theoretical setting, we consider the same Gaussian mixture and linear models for the generator and selector
. Let $w_*$ be a fixed unit vector in $\mathbb{R}^d$. The distribution $P_{orig}$ is
$$
x|y \sim N(y \tau w_*, I_d /d), \quad \text{for } y \in [-1, +1].
$$ Here, $\tau$ is a positive scalar that controls the overlap.

\paragraph{Synthesized Data Generation and Verification.} We sample $N_0$ data points from the distribution $P_{\text{orig}}$ to form the original dataset, which is used to train a linear model, $\hat{w}_{N_0}$, employing ordinary least squares. We then generate a large synthesized dataset using the trained model $\hat{w}_{N_0}$ with a sigmoid function, which undergoes selection by various verifiers into sets $D^{\text{gen}, N_0}_{ \text{vrf}}$. In our study, these verifiers are linear models parameterized by various $w_{\theta_{prune}}$, where $\theta_{prune}$ denotes the angle between the pruner and the ground truth $w_*$. 5Having the verified synthesized data, 
We then randomly select $n'$ data from $D^{\text{gen}, N_0}_{ \text{vrf}}$ to train and evaluate the final models. 

\subsection{Lessons Learned}

In Figure \ref{fig:simulation}, we conduct several simulations varying $N_0$, $\theta_{prune}$, and $n'$. A larger $N_0$ indicates a more effectively trained generator, utilizing a greater amount of original data. $\theta_{prune}$ determines the verifier's quality, with $\theta_{prune} = 0$ corresponding to the use of the ground truth as an oracle verifier. We further explore scenarios where $\theta_{prune} = \frac{\pi}{24}$ and $\frac{\pi}{12}$. The variable $n'$ represents the number of synthesized samples used to train the final model. We plot the scaling curves as $n'$ approaches infinity, a scenario analyzed by our theory. The ``random" line represents the process of randomly selecting $n'$ data points from the synthesized dataset without any verification. The ``clean" line indicates training of the model using $n'$ data points from the original distribution. We make the following observations:



\begin{figure}[tb]
\vspace{-20pt}
    \centering
    \includegraphics[width=\linewidth]{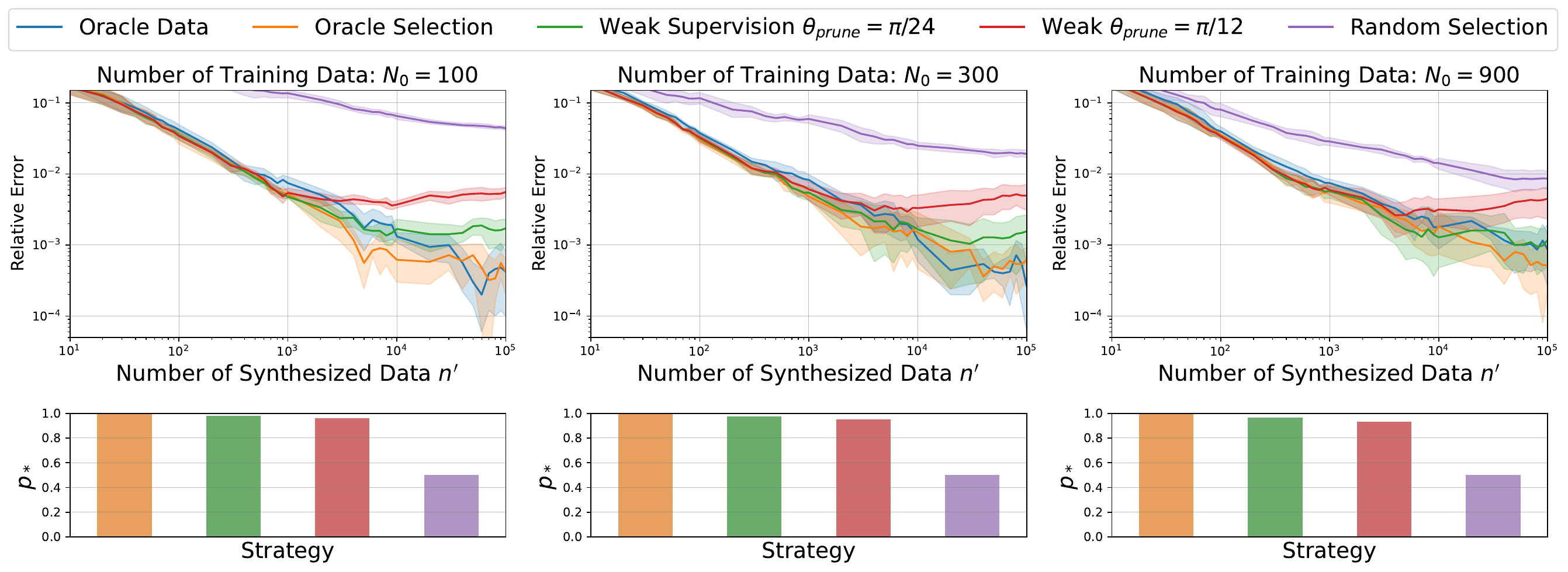}
    \caption{\small{{\bf{Simulations with Gaussian mixtures. (Top row)}} Relative error (accuracy relative to optimal accuracy) scaling as a function number of selected data, $n'$, used to train the model. $\tau=0.15, N_1=10^6$. The Bayes optimal classifier achieves approximately 94\% accuracy on this distribution. {\bf (Bottom row)}  $p_*$ values for all settings.}}
    \label{fig:simulation}
    \vspace{-8pt}
\end{figure}


\textbf{Oracle Verification Matches Training with Oracle Labels.} The oracle verifier achieves the best performance, matching training with clean data across all settings and attaining Bayes optimal accuracy. This validates the theory that synthesized data has the full information which can be extracted with verification.

\textbf{The Effectiveness of a Weak Verifier Depends on $n'$.} In practical scenarios, obtaining oracle-level verification is often challenging, so weaker forms of verification are typically employed. While a weak verifier generally leads to poorer performance, consistent with the decaying threshold points predicted by theory, there is a surprising ``sweet spot" at a certain data size (for the curves with $\theta_{prune}=\frac{\pi}{12}$). Training around 10,000 data outperforms training with more data. Therefore, the choice of verifier should also take into account the quantity of data selected. 

\textbf{$p_*$ is strongly correlated with the final performance.} Across all settings, the relative values of $p_*$ consistently align with the models' performance. The simulations provide practical validation of the theoretical results and are consistent with the interpretation presented in Section \ref{sec:thm-interpret}. Oracle supervision yields optimal performance, and, in general, improving both the generator and verifier leads to enhanced results. 



\vspace{-5pt}
\section{Experiments}
\label{sec:exp}

In both the theoretical results and the simulations, we examine a classification scenario where oracle supervision can select ``100\% correct" synthesized data. However, outside of these controlled settings, synthesized data may only approximate the ground truth, meaning the selected data might be closer to the truth rather than exactly correct. In this section, we explore verification for more general tasks using generative models through two experiments: (1) training a transformer to predict the eigenvalues of a matrix, and (2) fine-tuning Llama-2-7B on a news summarization task. Table \ref{tab:settings} provides an overview of how we gradually relax the setting from theoretical to empirical. For the purposes of our discussion here, we refer to model collapse whenever the performance of the model trained on synthetic data is worse than that of the original generator.



\begin{table}[t]
\small
    \centering
    \caption{\small{\textbf{Implementation of our three experiments}: We progressively explore our insights through three real-world experiments. First, we conduct simulations in a finite-data regime, where all settings align with theoretical expectations. Next, we examine transformers trained on generation tasks, evaluated using a 0-1 metric. Finally, we analyze large language models with general metrics.}}
    \setlength{\tabcolsep}{3.6pt}
    \begin{tabular}{l|ccccc}
    \toprule
        Settings & Data & Task & Model & Verifiers & Metric \\ \midrule
        Simulation & Gaussian mixture & classification & linear model & linear models & classification accuracy \\ 
        Arithmetic & synthesized data & generation & transformer & (noisy) verifiers & accuracy w. tolerance \\ 
        Summarization & XLSUM dataset & generation & Llama & Llamas & similarity w. Rouge-1 \\ 
    \bottomrule
    \end{tabular}
    \label{tab:settings}
    \vspace{-10pt}
\end{table}


        

\vspace{-5pt}
\subsection{Transformer for Arithmetic Tasks}
\label{subsec:exp-math}

Recall the setting described in Section \ref{sec:warmup}, where a transformer is trained in a generative manner using 200K samples of matrices and their corresponding eigenvalues. This setting allows for the synthesis of unlimited data by generating random matrices and using the generator to predict their eigenvalues. 

We first introduce the oracle \textbf{verifier} that measures the relative $L^1$ distance between the model's predictions and the correct solutions. We use greedy decoding to generate a prediction for each matrix, and only matrices with predictions within a tolerance of $\tau = 1\%$ (as determined by the oracle verifier) are retained. Additionally, we introduce a set of noisy verifiers, where data with predicted solutions exceeding the $1\%$ tolerance are still included with a probability of $p_{\text{noise}}$ \footnote{This is equivalent to introducing random noise in the selection of both correct and incorrect data.}. Since we are evaluating accuracy in a binary (0/1) manner for each data point, we can directly compute $\psi$ and $\phi$ in Equation \ref{equ:psiphi}, with $y' = y$ replaced by being within $1\%$ relative $L^1$ tolerance. The noisy verifier corresponds to setting $\phi = 1$ and $\psi = p_{\text{noise}}$. Consequently, according to our theory, the proxy is given by $p_* = 1 / (1 + p_{\text{noise}})$. For perfect oracle verifier, $p_{\text{noise}}=0$ and $p_*=1$; for random selection without verification, $p_{\text{noise}}=1$ and $p_*=0.5$.

We generate various verified synthesized datasets, with data sizes ranging from 1 million to 5 million, and up to 10 million samples. Using these datasets, the transformer is trained from scratch and evaluated with greedy decoding. We report how different verifiers, specifically different proxy values $p_*$, and varying dataset sizes affect the final performance in Figure \ref{fig:math_fpr}. We observe the following:

        

\begin{wrapfigure}{r}{0.4\textwidth}
    \centering
    \vspace{-2pt}
        \includegraphics[width=1\linewidth]{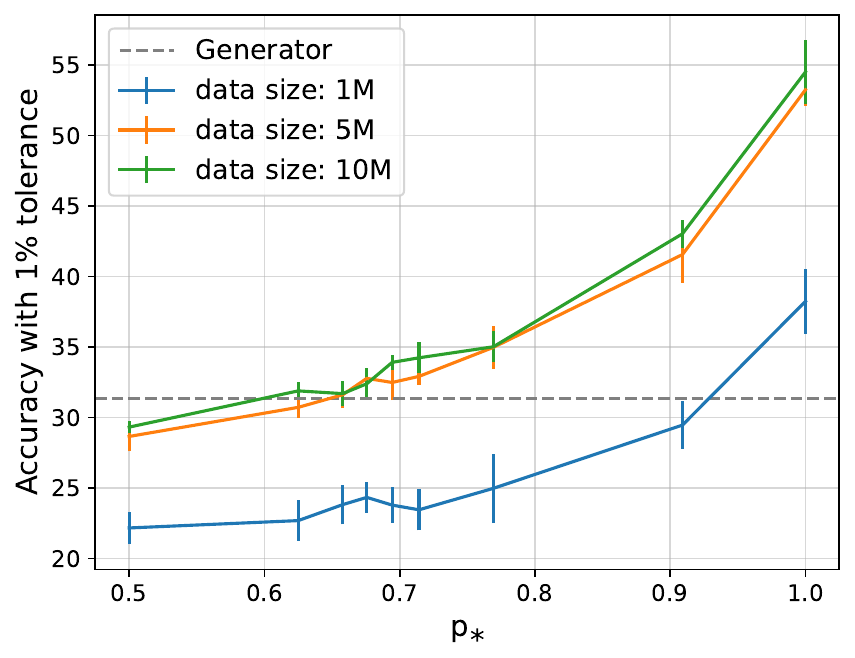}
         \vspace{-0.8cm}
    \caption{\small{{\bf Transformers computing eigenvalues.} Correlation between accuracy with $1\%$ tolerance, $p_*$, and the number of synthesized data. Model collapse is observed without verification ($p_* = 0.5$), while higher values of $p_*$ result in improved performance. Results are averaged over 5 seeds.}}
    \label{fig:math_fpr}
    \vspace{-20pt}
\end{wrapfigure}

\textbf{Model Collapse is Observed.} When using synthesized data without any verification, at $p_* = 0.5$, even with a significantly larger dataset — 10 million samples, or 50 times the size of the generator's original training set — the trained model performs worse than the generator itself (represented by the dashed line). This outcome indicates the occurrence of model collapse.

\textbf{$p_*$ as a Reliable Proxy for Final Performance.} When leveraging verification, we observe a consistent increase in performance. With the oracle verifier ($p_* = 1$), the accuracy nearly doubles. This observation aligns with practices in code generation and mathematics \citep{haluptzok2022language, trinh2024solving}, where a natural verifier, such as precomputed solutions, unit tests, or a compiler, is available. Across all dataset sizes, $p_*$ shows a strong correlation with the final performance. Furthermore, using 10 million samples approximates the effect of having infinite synthetic data; the turning point on the green curve, which surpasses the generator's performance, occurs around 0.65 in Figure \ref{fig:math_fpr}. This value is close to the generator's error rate (i.e., 1 minus the accuracy, represented by the dashed line)  and reflects the phase transition predicted by our theory.

\vspace{-5pt}
\subsection{LLMs for News Summarization}\label{sec:sub-news}
We now turn to one of the most standard tasks in NLP:  news summarization. We utilize the English summarization subset of the XLSUM dataset \citep{hasan-etal-2021-xl}, the largest publicly available summarization dataset, consisting of 307,000 training samples and 11,500 test samples. Each sample in this dataset contains a news article $x$ paired with a professionally annotated summary $y$. Unlike in the previous cases, we no longer have infinite $(x, y)$ pairs due to the finite number of news articles available. Therefore, we fine-tune a Llama-2-7B model \citep{touvron2023llama} on only $12.5\%$ of the training set to serve as the generator. The synthesized data (news summaries) is generated using the articles of the entire training set with greedy decoding. This approach reflects real-world conditions where synthetic data generation can significantly outpace human annotation. 

In this setting, we use the Rouge-1 score \citep{lin2004rouge} to evaluate the generated summary $y'$. Rouge-1 assesses the quality of the generated summary by measuring the overlap of individual words between the generated and human-written summaries $y$. Similar to other tasks in natural language processing (NLP), we only have a similarity measure between $y'$ and $y$, rather than a binary 0/1 measure.  In Appendix \ref{sec:app_gen_proxy}, we generalize the definitions of $\psi$ and $\phi$ for these metrics. The core idea is to extend $y' = y$ to a similarity score in the range  $[0, 1]$.



We consider three selection strategies: (1) \textbf{Selection with Oracle}: We calculate the Rouge score between the generated summary and the ground truth summary, keeping the data with the highest scores; (2) \textbf{Selection with Weak Supervision}: We leverage a fine-tuned Llama-3 model with higher Rouge score than the generator and keep the data with the lowest perplexity; (3) \textbf{Self-Selection}: We use the generator to keep the data with the lowest perplexity. We apply three selection rates: $12.5\%$ (when the selected synthesized data is the same size as the original data used to train the generator), $25\%$, and $50\%$. For each combination, we collect the selected synthesized data to finetune the Llama-2 model. We present scaling law curves that illustrate how the Rouge-1 score improves with increasing amount of selected synthesized data used for training. Throughout the experiments, all finetuning was performed with full parameter training. Details are provided in \cref{sec:news_append}. 

The results and $p_*$ are shown in Figure \ref{fig:llama2_new}. We observe the following:

\textbf{Model Collapse is Observed.} The Random Selection curve represent training with synthesized data directly without verification. In Figure \ref{fig:llama2_new} \textbf{Left}, using the same amount of synthesized data results in worse performance compared to using the original data, indicating model collapse (comparing `Random Selection' with `Generator'). Only with more data, the Random Selection lines improve and nearly match the performance of the generator. 

\textbf{Selection by Oracle Performs Best.} The dataset curated with oracle selection surpasses the performance of the generator in all settings. Oracle selection with  12.5\% of the data selected even surpasses the model trained with 100\% of the training set and original labels. The oracle verifier consistently has the best $p_*$ compared with the other methods.

\textbf{$p_*$ Strongly Correlates with Performance.} Surprisingly, self-selection leads to better performance than the generator, while Llama-3 verification results in performance similar to random selection. These findings seem counterintuitive, {given that Llama-3 actually achieves a better Rouge-1 score}. However, they are consistent with the measured $p_*$. Intuitively, self-selection tends to favor easy-to-learn samples, resulting in good performance with fewer data points. Although Llama-3 can generate better summaries, the synthesized data are produced by Llama-2, which is less correlated with Llama-3. As discussed in Equation \ref{eq:principal-angles}, $p_*$ depends on three correlations and Llama-3 is not necessarily a better verifier in this context. This comparison highlights the challenges of selecting an appropriate verifier when a fixed oracle, like those used in code or math tasks, is not available. In such cases, our proposed $p_*$ can serve as a valuable proxy before any training is conducted.


\begin{figure}[tb]
\vspace{-20pt}
    \centering
    \includegraphics[width=\linewidth]{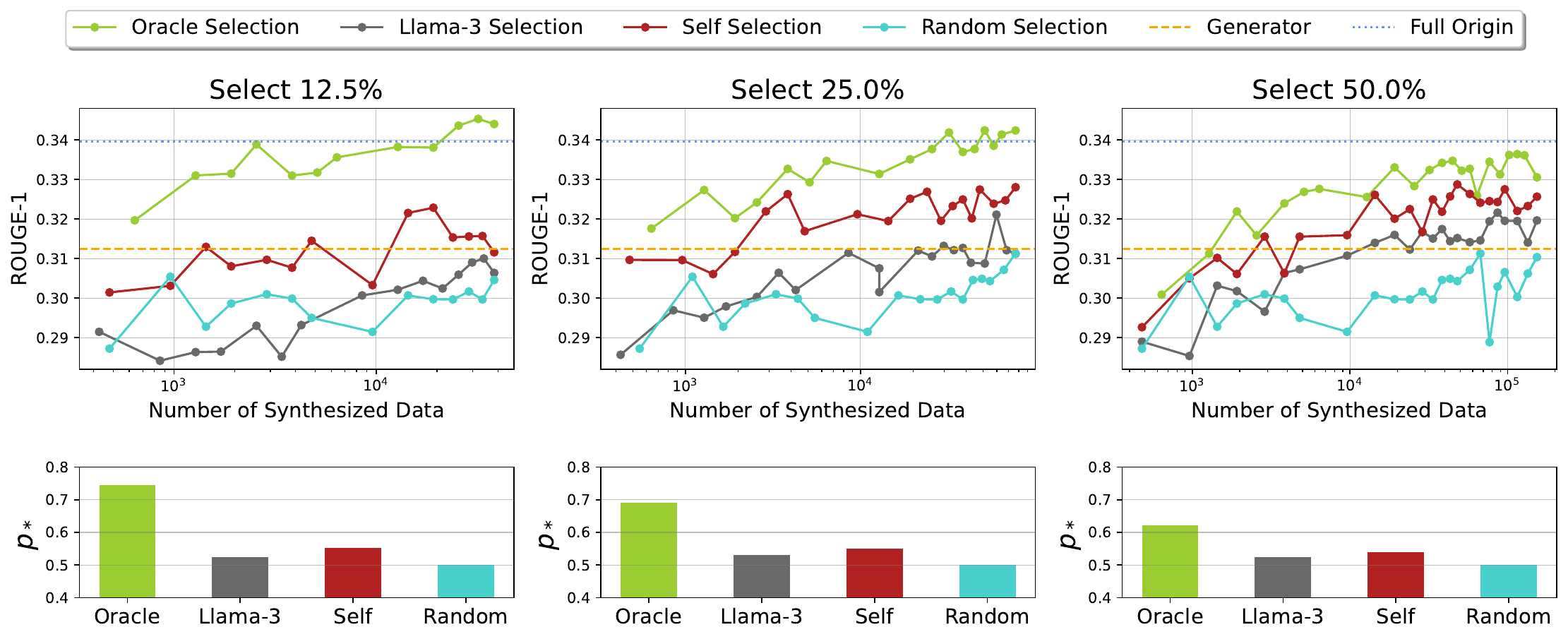}
    \caption{\small{{\bf News summarization with LLMs.} \textbf{(Top Row)} The three figures from left to right represents models trained on 12.5\%, 25\%, and 50\% of the selected data individually. Each figure includes four curves illustrating different training scenarios: (1) selection with oracle, (2) selection with Llama-3 as a weak supervision, (3) self-selection, and (4) random selection. Additionally, two horizontal lines are included for comparison: one representing the generator model and the other representing a model trained with 100\% data with original labels. \textbf{(Bottom row)} Computed values of $p_*$ for the verifiers for corresponding proportions of selected data. Training with data selected by a verifier with higher $p_*$ achieves better performance. }}
    \label{fig:llama2_new}
    \vspace{-15pt}
\end{figure}





\vspace{-5pt}
\section{Discussion and Limitations}
\label{sec:conc}


In this paper, we consider a novel problem related to synthesized data: how to prevent model collapse through data selection. We propose to leverage a verifier to improve the synthesized data. We emphasize it is crucial to focus not only on the quality of the generator but also on having a high-quality verifier to select the data. We theoretically show that verification is all you need for synthesized data and identify a proxy function for performance after data selection. Through three solid experiments, we demonstrate that a decent selector can prevent model collapse and our proxy function is a reliable measure. Our work is of significant theoretical and practical importance in the era of large models with increasing use of generated data.

Can a model improve itself? 
In our paper, we present results showing that a math transformer using beam search does not improve test accuracy, while Llama-2, through self-selection of its generated data, can yield a model that performs better than the original generator. In the first experiment, the model selects a better $y$ (output) for each $x$ (input) (label selection). However, at the distribution level, the selected $y$ are not better. In the second experiment, the model selects better $(x, y)$ pairs from a larger pool of news articles ($x$) than available in the generator's training set. This introduces new information through the $x$ and results in a shift in the distribution of $x$. Due to this distribution shift, the trained model can outperform the original generator even with self-selection.

One limitation of this study is that we only considered data selection as a means of data curation. Besides data selection, data curation also includes methods such as data augmentation, data regeneration, and weighting. The exploration of general data curation methods to avoid model collapse is left for future work. 

\section*{Acknowledgements}

YF and JK acknowledge support through NSF NRT training grant award 1922658. Part of this work was done while JK and YF  were hosted by the Centre Sciences de Donnees (CSD) at the École Normale Supérieure (ENS) in 2023/24, and JK and YF would like to thank CSD and ENS for their hospitality. YF and PY would like to thank Yanzhu Guo, Di He, Zhenyu He for
discussions and suggestions.
This work was supported in part through the NYU IT High Performance Computing resources, services, and staff expertise.

\clearpage

\bibliography{main}

\begin{thebibliography}{69}
\providecommand{\natexlab}[1]{#1}
\providecommand{\url}[1]{\texttt{#1}}
\expandafter\ifx\csname urlstyle\endcsname\relax
  \providecommand{\doi}[1]{doi: #1}\else
  \providecommand{\doi}{doi: \begingroup \urlstyle{rm}\Url}\fi

\bibitem[Achiam et~al.(2023)Achiam, Adler, Agarwal, Ahmad, Akkaya, Aleman, Almeida, Altenschmidt, Altman, Anadkat, et~al.]{achiam2023gpt}
Josh Achiam, Steven Adler, Sandhini Agarwal, Lama Ahmad, Ilge Akkaya, Florencia~Leoni Aleman, Diogo Almeida, Janko Altenschmidt, Sam Altman, Shyamal Anadkat, et~al.
\newblock Gpt-4 technical report.
\newblock \emph{arXiv preprint arXiv:2303.08774}, 2023.

\bibitem[Albalak et~al.(2024)Albalak, Elazar, Xie, Longpre, Lambert, Wang, Muennighoff, Hou, Pan, Jeong, et~al.]{albalak2024survey}
Alon Albalak, Yanai Elazar, Sang~Michael Xie, Shayne Longpre, Nathan Lambert, Xinyi Wang, Niklas Muennighoff, Bairu Hou, Liangming Pan, Haewon Jeong, et~al.
\newblock A survey on data selection for language models.
\newblock \emph{arXiv preprint arXiv:2402.16827}, 2024.

\bibitem[Alemohammad et~al.(2023)Alemohammad, Casco-Rodriguez, Luzi, Humayun, Babaei, LeJeune, Siahkoohi, and Baraniuk]{alemohammad2023selfconsuming}
Sina Alemohammad, Josue Casco-Rodriguez, Lorenzo Luzi, Ahmed~Imtiaz Humayun, Hossein Babaei, Daniel LeJeune, Ali Siahkoohi, and Richard~G. Baraniuk.
\newblock Self-consuming generative models go mad.
\newblock \emph{arXiv preprint arxiv:2307.01850}, 2023.

\bibitem[Alemohammad et~al.(2024)Alemohammad, Humayun, Agarwal, Collomosse, and Baraniuk]{alemohammad2024selfimprovingdiffusionmodelssynthetic}
Sina Alemohammad, Ahmed~Imtiaz Humayun, Shruti Agarwal, John Collomosse, and Richard Baraniuk.
\newblock Self-improving diffusion models with synthetic data, 2024.
\newblock URL \url{https://arxiv.org/abs/2408.16333}.

\bibitem[Allen-Zhu \& Li(2022)Allen-Zhu and Li]{allen2022towards}
Zeyuan Allen-Zhu and Yuanzhi Li.
\newblock Towards understanding ensemble, knowledge distillation and self-distillation in deep learning.
\newblock In \emph{The Eleventh International Conference on Learning Representations}, 2022.

\bibitem[Azizi et~al.(2023)Azizi, Kornblith, Saharia, Norouzi, and Fleet]{azizi2023synthetic}
Shekoofeh Azizi, Simon Kornblith, Chitwan Saharia, Mohammad Norouzi, and David~J. Fleet.
\newblock Synthetic data from diffusion models improves imagenet classification.
\newblock \emph{Transactions on Machine Learning Research}, 2023.
\newblock ISSN 2835-8856.

\bibitem[Bai et~al.(2022)Bai, Kadavath, Kundu, Askell, Kernion, Jones, Chen, Goldie, Mirhoseini, McKinnon, et~al.]{bai2022constitutional}
Yuntao Bai, Saurav Kadavath, Sandipan Kundu, Amanda Askell, Jackson Kernion, Andy Jones, Anna Chen, Anna Goldie, Azalia Mirhoseini, Cameron McKinnon, et~al.
\newblock Constitutional ai: Harmlessness from ai feedback.
\newblock \emph{arXiv preprint arXiv:2212.08073}, 2022.

\bibitem[Bertrand et~al.(2023)Bertrand, Bose, Duplessis, Jiralerspong, and Gidel]{bertrand2023stability}
Quentin Bertrand, Avishek~Joey Bose, Alexandre Duplessis, Marco Jiralerspong, and Gauthier Gidel.
\newblock On the stability of iterative retraining of generative models on their own data.
\newblock \emph{arXiv preprint arxiv:2310.00429}, 2023.

\bibitem[Bohacek \& Farid(2023)Bohacek and Farid]{bohacek2023nepotistically}
Matyas Bohacek and Hany Farid.
\newblock Nepotistically trained generative-ai models collapse, 2023.

\bibitem[Briesch et~al.(2023)Briesch, Sobania, and Rothlauf]{briesch2023large}
Martin Briesch, Dominik Sobania, and Franz Rothlauf.
\newblock Large language models suffer from their own output: An analysis of the self-consuming training loop, 2023.

\bibitem[Burg et~al.(2023)Burg, Wenzel, Zietlow, Horn, Makansi, Locatello, and Russell]{burg2023image}
Max~F Burg, Florian Wenzel, Dominik Zietlow, Max Horn, Osama Makansi, Francesco Locatello, and Chris Russell.
\newblock Image retrieval outperforms diffusion models on data augmentation.
\newblock \emph{Transactions on Machine Learning Research}, 2023.
\newblock ISSN 2835-8856.

\bibitem[Charton(2022)]{charton2022linear}
Fran\c{c}ois Charton.
\newblock Linear algebra with transformers.
\newblock \emph{Transactions on Machine Learning Research}, 2022.
\newblock ISSN 2835-8856.
\newblock URL \url{https://openreview.net/forum?id=Hp4g7FAXXG}.

\bibitem[Charton(2024)]{charton2023transformers}
Fran{\c{c}}ois Charton.
\newblock Learning the greatest common divisor: explaining transformer predictions.
\newblock In \emph{The Twelfth International Conference on Learning Representations}, 2024.
\newblock URL \url{https://openreview.net/forum?id=cmcD05NPKa}.

\bibitem[Das \& Sanghavi(2023)Das and Sanghavi]{das2023understanding}
Rudrajit Das and Sujay Sanghavi.
\newblock Understanding self-distillation in the presence of label noise.
\newblock In \emph{International Conference on Machine Learning}, pp.\  7102--7140. PMLR, 2023.

\bibitem[Dohmatob et~al.(2024{\natexlab{a}})Dohmatob, Feng, and Kempe]{dohmatob2024model}
Elvis Dohmatob, Yunzhen Feng, and Julia Kempe.
\newblock Model collapse demystified: The case of regression.
\newblock \emph{arXiv preprint arXiv:2402.07712}, 2024{\natexlab{a}}.

\bibitem[Dohmatob et~al.(2024{\natexlab{b}})Dohmatob, Feng, Yang, Charton, and Kempe]{dohmatob2024tale}
Elvis Dohmatob, Yunzhen Feng, Pu~Yang, Fran{\c{c}}ois Charton, and Julia Kempe.
\newblock A tale of tails: Model collapse as a change of scaling laws.
\newblock In \emph{Forty-first International Conference on Machine Learning}, 2024{\natexlab{b}}.
\newblock URL \url{https://openreview.net/forum?id=KVvku47shW}.

\bibitem[Dong et~al.(2019)Dong, Hou, Lu, and Zhang]{dong2019distillation}
Bin Dong, Jikai Hou, Yiping Lu, and Zhihua Zhang.
\newblock Distillation $\sim$ early stopping? harvesting dark knowledge utilizing anisotropic information retrieval for overparameterized neural network.
\newblock \emph{arXiv preprint arXiv:1910.01255}, 2019.

\bibitem[Dunlap et~al.(2023)Dunlap, Umino, Zhang, Yang, Gonzalez, and Darrell]{dunlap2024diversify}
Lisa Dunlap, Alyssa Umino, Han Zhang, Jiezhi Yang, Joseph~E Gonzalez, and Trevor Darrell.
\newblock Diversify your vision datasets with automatic diffusion-based augmentation.
\newblock \emph{Advances in Neural Information Processing Systems}, 36, 2023.

\bibitem[Eldan \& Li(2023)Eldan and Li]{eldan2023tinystories}
Ronen Eldan and Yuanzhi Li.
\newblock Tinystories: How small can language models be and still speak coherent english?
\newblock \emph{arXiv preprint arXiv:2305.07759}, 2023.

\bibitem[Furlanello et~al.(2018)Furlanello, Lipton, Tschannen, Itti, and Anandkumar]{furlanello2018born}
Tommaso Furlanello, Zachary Lipton, Michael Tschannen, Laurent Itti, and Anima Anandkumar.
\newblock Born again neural networks.
\newblock In \emph{International conference on machine learning}, pp.\  1607--1616. PMLR, 2018.

\bibitem[Garg et~al.(2022)Garg, Tsipras, Liang, and Valiant]{garg2022can}
Shivam Garg, Dimitris Tsipras, Percy~S Liang, and Gregory Valiant.
\newblock What can transformers learn in-context? a case study of simple function classes.
\newblock \emph{Advances in Neural Information Processing Systems}, 35:\penalty0 30583--30598, 2022.

\bibitem[Gillman et~al.(2024)Gillman, Freeman, Aggarwal, Chia-Hong, Luo, Tian, and Sun]{gillmanself}
Nate Gillman, Michael Freeman, Daksh Aggarwal, HSU Chia-Hong, Calvin Luo, Yonglong Tian, and Chen Sun.
\newblock Self-correcting self-consuming loops for generative model training.
\newblock In \emph{Forty-first International Conference on Machine Learning}, 2024.

\bibitem[Gunasekar et~al.(2023)Gunasekar, Zhang, Aneja, Mendes, Del~Giorno, Gopi, Javaheripi, Kauffmann, de~Rosa, Saarikivi, et~al.]{gunasekar2023textbooks}
Suriya Gunasekar, Yi~Zhang, Jyoti Aneja, Caio C{\'e}sar~Teodoro Mendes, Allie Del~Giorno, Sivakanth Gopi, Mojan Javaheripi, Piero Kauffmann, Gustavo de~Rosa, Olli Saarikivi, et~al.
\newblock Textbooks are all you need.
\newblock \emph{arXiv preprint arXiv:2306.11644}, 2023.

\bibitem[Guo et~al.(2023)Guo, Shang, Vazirgiannis, and Clavel]{guo2023curious}
Yanzhu Guo, Guokan Shang, Michalis Vazirgiannis, and Chloé Clavel.
\newblock The curious decline of linguistic diversity: Training language models on synthetic text, 2023.

\bibitem[Haluptzok et~al.(2022)Haluptzok, Bowers, and Kalai]{haluptzok2022language}
Patrick Haluptzok, Matthew Bowers, and Adam~Tauman Kalai.
\newblock Language models can teach themselves to program better.
\newblock In \emph{The Eleventh International Conference on Learning Representations}, 2022.

\bibitem[Hasan et~al.(2021)Hasan, Bhattacharjee, Islam, Mubasshir, Li, Kang, Rahman, and Shahriyar]{hasan-etal-2021-xl}
Tahmid Hasan, Abhik Bhattacharjee, Md.~Saiful Islam, Kazi Mubasshir, Yuan-Fang Li, Yong-Bin Kang, M.~Sohel Rahman, and Rifat Shahriyar.
\newblock {XL}-sum: Large-scale multilingual abstractive summarization for 44 languages.
\newblock In \emph{Findings of the Association for Computational Linguistics: ACL-IJCNLP 2021}, pp.\  4693--4703, Online, August 2021. Association for Computational Linguistics.
\newblock URL \url{https://aclanthology.org/2021.findings-acl.413}.

\bibitem[Hataya et~al.(2023)Hataya, Bao, and Arai]{Hataya_2023_ICCV}
Ryuichiro Hataya, Han Bao, and Hiromi Arai.
\newblock Will large-scale generative models corrupt future datasets?
\newblock In \emph{Proceedings of the IEEE/CVF International Conference on Computer Vision (ICCV)}, pp.\  20555--20565, October 2023.

\bibitem[He et~al.(2023)He, Sun, Yu, Xue, Zhang, Torr, Bai, and QI]{he2023is}
Ruifei He, Shuyang Sun, Xin Yu, Chuhui Xue, Wenqing Zhang, Philip Torr, Song Bai, and XIAOJUAN QI.
\newblock {IS} {SYNTHETIC} {DATA} {FROM} {GENERATIVE} {MODELS} {READY} {FOR} {IMAGE} {RECOGNITION}?
\newblock In \emph{The Eleventh International Conference on Learning Representations}, 2023.
\newblock URL \url{https://openreview.net/forum?id=nUmCcZ5RKF}.

\bibitem[Hemmat et~al.(2023)Hemmat, Pezeshki, Bordes, Drozdzal, and Romero-Soriano]{hemmat2023feedback}
Reyhane~Askari Hemmat, Mohammad Pezeshki, Florian Bordes, Michal Drozdzal, and Adriana Romero-Soriano.
\newblock Feedback-guided data synthesis for imbalanced classification.
\newblock \emph{arXiv preprint arXiv:2310.00158}, 2023.

\bibitem[Hinton et~al.(2015)Hinton, Vinyals, and Dean]{hinton2015distilling}
Geoffrey Hinton, Oriol Vinyals, and Jeff Dean.
\newblock Distilling the knowledge in a neural network.
\newblock \emph{arXiv preprint arXiv:1503.02531}, 2015.

\bibitem[Hoffmann et~al.(2022)Hoffmann, Borgeaud, Mensch, Buchatskaya, Cai, Rutherford, de~Las~Casas, Hendricks, Welbl, Clark, Hennigan, Noland, Millican, van~den Driessche, Damoc, Guy, Osindero, Simonyan, Elsen, Rae, Vinyals, and Sifre]{hoffmann2022trainingChinchilla}
Jordan Hoffmann, Sebastian Borgeaud, Arthur Mensch, Elena Buchatskaya, Trevor Cai, Eliza Rutherford, Diego de~Las~Casas, Lisa~Anne Hendricks, Johannes Welbl, Aidan Clark, Tom Hennigan, Eric Noland, Katie Millican, George van~den Driessche, Bogdan Damoc, Aurelia Guy, Simon Osindero, Karen Simonyan, Erich Elsen, Jack~W. Rae, Oriol Vinyals, and Laurent Sifre.
\newblock Training compute-optimal large language models, 2022.

\bibitem[Kakade et~al.(2008)Kakade, Sridharan, and Tewari]{kakade2008complexity}
Sham~M Kakade, Karthik Sridharan, and Ambuj Tewari.
\newblock On the complexity of linear prediction: Risk bounds, margin bounds, and regularization.
\newblock \emph{Advances in neural information processing systems}, 21, 2008.

\bibitem[Kaplan et~al.(2020)Kaplan, McCandlish, Henighan, Brown, Chess, Child, Gray, Radford, Wu, and Amodei]{kaplan2020scaling}
Jared Kaplan, Sam McCandlish, Tom Henighan, Tom~B. Brown, Benjamin Chess, Rewon Child, Scott Gray, Alec Radford, Jeffrey Wu, and Dario Amodei.
\newblock Scaling laws for neural language models.
\newblock \emph{arXiv preprint arXiv:2001.08361}, 2020.

\bibitem[Kingma \& Ba(2014)Kingma and Ba]{kingma2014adam}
Diederik~P Kingma and Jimmy Ba.
\newblock Adam: A method for stochastic optimization.
\newblock \emph{arXiv preprint arXiv:1412.6980}, 2014.

\bibitem[Kirillov et~al.(2023)Kirillov, Mintun, Ravi, Mao, Rolland, Gustafson, Xiao, Whitehead, Berg, Lo, et~al.]{kirillov2023segment}
Alexander Kirillov, Eric Mintun, Nikhila Ravi, Hanzi Mao, Chloe Rolland, Laura Gustafson, Tete Xiao, Spencer Whitehead, Alexander~C Berg, Wan-Yen Lo, et~al.
\newblock Segment anything.
\newblock In \emph{Proceedings of the IEEE/CVF International Conference on Computer Vision}, pp.\  4015--4026, 2023.

\bibitem[Kolossov et~al.(2024)Kolossov, Montanari, and Tandon]{kolossov2024towards}
Germain Kolossov, Andrea Montanari, and Pulkit Tandon.
\newblock Towards a statistical theory of data selection under weak supervision.
\newblock In \emph{The Twelfth International Conference on Learning Representations}, 2024.
\newblock URL \url{https://openreview.net/forum?id=HhfcNgQn6p}.

\bibitem[LeBrun et~al.(2021)LeBrun, Sordoni, and O'Donnell]{lebrun2021evaluating}
Benjamin LeBrun, Alessandro Sordoni, and Timothy~J O'Donnell.
\newblock Evaluating distributional distortion in neural language modeling.
\newblock In \emph{International Conference on Learning Representations}, 2021.

\bibitem[Lee et~al.(2023)Lee, Phatale, Mansoor, Lu, Mesnard, Bishop, Carbune, and Rastogi]{lee2023rlaif}
Harrison Lee, Samrat Phatale, Hassan Mansoor, Kellie Lu, Thomas Mesnard, Colton Bishop, Victor Carbune, and Abhinav Rastogi.
\newblock Rlaif: Scaling reinforcement learning from human feedback with ai feedback.
\newblock \emph{arXiv preprint arXiv:2309.00267}, 2023.

\bibitem[Li et~al.(2022)Li, Li, Xiong, and Hoi]{li2022blip}
Junnan Li, Dongxu Li, Caiming Xiong, and Steven Hoi.
\newblock Blip: Bootstrapping language-image pre-training for unified vision-language understanding and generation.
\newblock In \emph{International conference on machine learning}, pp.\  12888--12900. PMLR, 2022.

\bibitem[Lin(2004)]{lin2004rouge}
Chin-Yew Lin.
\newblock Rouge: A package for automatic evaluation of summaries.
\newblock In \emph{Text summarization branches out}, pp.\  74--81, 2004.

\bibitem[Martínez et~al.(2023{\natexlab{a}})Martínez, Watson, Reviriego, Hernández, Juarez, and Sarkar]{martínez2023combining}
Gonzalo Martínez, Lauren Watson, Pedro Reviriego, José~Alberto Hernández, Marc Juarez, and Rik Sarkar.
\newblock Combining generative artificial intelligence (ai) and the internet: Heading towards evolution or degradation?
\newblock \emph{arXiv preprint arxiv: 2303.01255}, 2023{\natexlab{a}}.

\bibitem[Martínez et~al.(2023{\natexlab{b}})Martínez, Watson, Reviriego, Hernández, Juarez, and Sarkar]{martínez2023understanding}
Gonzalo Martínez, Lauren Watson, Pedro Reviriego, José~Alberto Hernández, Marc Juarez, and Rik Sarkar.
\newblock Towards understanding the interplay of generative artificial intelligence and the internet.
\newblock \emph{arXiv preprint arxiv: 2306.06130}, 2023{\natexlab{b}}.

\bibitem[McAllester(2003)]{mcAllester2003}
David McAllester.
\newblock Simplified pac-bayesian margin bounds.
\newblock In \emph{Learning Theory and Kernel Machines}. Springer Berlin Heidelberg, 2003.

\bibitem[Mobahi et~al.(2020)Mobahi, Farajtabar, and Bartlett]{mobahi2020self}
Hossein Mobahi, Mehrdad Farajtabar, and Peter Bartlett.
\newblock Self-distillation amplifies regularization in hilbert space.
\newblock \emph{Advances in Neural Information Processing Systems}, 33:\penalty0 3351--3361, 2020.

\bibitem[OpenAI(2024)]{Sora}
OpenAI.
\newblock Video generation models as world simulators.
\newblock \url{https://openai.com/index/video-generation-models-as-world-simulators/}, 2024.

\bibitem[Ouyang et~al.(2022)Ouyang, Wu, Jiang, Almeida, Wainwright, Mishkin, Zhang, Agarwal, Slama, Ray, et~al.]{ouyang2022training}
Long Ouyang, Jeffrey Wu, Xu~Jiang, Diogo Almeida, Carroll Wainwright, Pamela Mishkin, Chong Zhang, Sandhini Agarwal, Katarina Slama, Alex Ray, et~al.
\newblock Training language models to follow instructions with human feedback.
\newblock \emph{Advances in neural information processing systems}, 35:\penalty0 27730--27744, 2022.

\bibitem[Padmakumar \& He(2024)Padmakumar and He]{padmakumar2024writing}
Vishakh Padmakumar and He~He.
\newblock Does writing with language models reduce content diversity?
\newblock In \emph{International Conference on Learning Representations (ICLR)}, 2024.

\bibitem[Peng et~al.(2023)Peng, Li, He, Galley, and Gao]{peng2023instruction}
Baolin Peng, Chunyuan Li, Pengcheng He, Michel Galley, and Jianfeng Gao.
\newblock Instruction tuning with gpt-4.
\newblock \emph{arXiv preprint arXiv:2304.03277}, 2023.

\bibitem[Power et~al.(2022)Power, Burda, Edwards, Babuschkin, and Misra]{power2022grokking}
Alethea Power, Yuri Burda, Harri Edwards, Igor Babuschkin, and Vedant Misra.
\newblock Grokking: Generalization beyond overfitting on small algorithmic datasets.
\newblock \emph{arXiv preprint arXiv:2201.02177}, 2022.

\bibitem[Ramesh et~al.(2021)Ramesh, Pavlov, Goh, Gray, Voss, Radford, Chen, and Sutskever]{pmlr-v139-ramesh21a}
Aditya Ramesh, Mikhail Pavlov, Gabriel Goh, Scott Gray, Chelsea Voss, Alec Radford, Mark Chen, and Ilya Sutskever.
\newblock Zero-shot text-to-image generation.
\newblock In Marina Meila and Tong Zhang (eds.), \emph{Proceedings of the 38th International Conference on Machine Learning}, volume 139 of \emph{Proceedings of Machine Learning Research}, pp.\  8821--8831. PMLR, 18--24 Jul 2021.

\bibitem[Rombach et~al.(2022)Rombach, Blattmann, Lorenz, Esser, and Ommer]{Rombach_2022_CVPR}
Robin Rombach, Andreas Blattmann, Dominik Lorenz, Patrick Esser, and Bj\"orn Ommer.
\newblock High-resolution image synthesis with latent diffusion models.
\newblock In \emph{Proceedings of the IEEE/CVF Conference on Computer Vision and Pattern Recognition (CVPR)}, pp.\  10684--10695, June 2022.

\bibitem[Seddik et~al.(2024)Seddik, Chen, Hayou, Youssef, and Debbah]{seddik2024bad}
Mohamed El~Amine Seddik, Suei-Wen Chen, Soufiane Hayou, Pierre Youssef, and Merouane Debbah.
\newblock How bad is training on synthetic data? a statistical analysis of language model collapse.
\newblock \emph{arXiv preprint arXiv:2404.05090}, 2024.

\bibitem[Setlur et~al.(2024)Setlur, Garg, Geng, Garg, Smith, and Kumar]{setlur2024rl}
Amrith Setlur, Saurabh Garg, Xinyang Geng, Naman Garg, Virginia Smith, and Aviral Kumar.
\newblock Rl on incorrect synthetic data scales the efficiency of llm math reasoning by eight-fold.
\newblock \emph{arXiv preprint arXiv:2406.14532}, 2024.

\bibitem[Shalev-Shwartz \& Ben-David(2014)Shalev-Shwartz and Ben-David]{BenDavidUnderstanding}
Shai Shalev-Shwartz and Shai Ben-David.
\newblock \emph{Understanding Machine Learning - From Theory to Algorithms.}
\newblock Cambridge University Press, 2014.

\bibitem[Shumailov et~al.(2023)Shumailov, Shumaylov, Zhao, Gal, Papernot, and Anderson]{shumailov2023curse}
Ilia Shumailov, Zakhar Shumaylov, Yiren Zhao, Yarin Gal, Nicolas Papernot, and Ross Anderson.
\newblock The curse of recursion: Training on generated data makes models forget.
\newblock \emph{arXiv preprint arxiv:2305.17493}, 2023.

\bibitem[Shumailov et~al.(2024)Shumailov, Shumaylov, Zhao, Gal, Papernot, and Anderson]{Shumailov2024Nature}
Ilia Shumailov, Zakhar Shumaylov, Yiren Zhao, Yarin Gal, Nicolas Papernot, and Ross Anderson.
\newblock Ai models collapse when trained on recursively generated data.
\newblock \emph{Nature}, 631, 2024.

\bibitem[Sorscher et~al.(2022)Sorscher, Geirhos, Shekhar, Ganguli, and Morcos]{sorscher2022beyond}
Ben Sorscher, Robert Geirhos, Shashank Shekhar, Surya Ganguli, and Ari Morcos.
\newblock Beyond neural scaling laws: beating power law scaling via data pruning.
\newblock \emph{Advances in Neural Information Processing Systems}, 35:\penalty0 19523--19536, 2022.

\bibitem[Touvron et~al.(2023)Touvron, Martin, Stone, Albert, Almahairi, Babaei, Bashlykov, Batra, Bhargava, Bhosale, et~al.]{touvron2023llama}
Hugo Touvron, Louis Martin, Kevin Stone, Peter Albert, Amjad Almahairi, Yasmine Babaei, Nikolay Bashlykov, Soumya Batra, Prajjwal Bhargava, Shruti Bhosale, et~al.
\newblock Llama 2: Open foundation and fine-tuned chat models.
\newblock \emph{arXiv preprint arXiv:2307.09288}, 2023.

\bibitem[Trinh et~al.(2024)Trinh, Wu, Le, He, and Luong]{trinh2024solving}
Trieu~H Trinh, Yuhuai Wu, Quoc~V Le, He~He, and Thang Luong.
\newblock Solving olympiad geometry without human demonstrations.
\newblock \emph{Nature}, 625\penalty0 (7995):\penalty0 476--482, 2024.

\bibitem[Ulmer et~al.(2024)Ulmer, Mansimov, Lin, Sun, Gao, and Zhang]{ulmer2024bootstrapping}
Dennis Ulmer, Elman Mansimov, Kaixiang Lin, Justin Sun, Xibin Gao, and Yi~Zhang.
\newblock Bootstrapping llm-based task-oriented dialogue agents via self-talk.
\newblock \emph{arXiv preprint arXiv:2401.05033}, 2024.

\bibitem[Um et~al.(2024)Um, Lee, and Ye]{um2024dont}
Soobin Um, Suhyeon Lee, and Jong~Chul Ye.
\newblock Don't play favorites: Minority guidance for diffusion models.
\newblock In \emph{The Twelfth International Conference on Learning Representations}, 2024.
\newblock URL \url{https://openreview.net/forum?id=3NmO9lY4Jn}.

\bibitem[Vaswani et~al.(2017)Vaswani, Shazeer, Parmar, Uszkoreit, Jones, Gomez, Kaiser, and Polosukhin]{transformer17}
Ashish Vaswani, Noam Shazeer, Niki Parmar, Jakob Uszkoreit, Llion Jones, Aidan~N. Gomez, Lukasz Kaiser, and Illia Polosukhin.
\newblock Attention is all you need.
\newblock In \emph{Advances in Neural Information Processing Systems}, pp.\  6000--6010, 2017.

\bibitem[Wang et~al.(2023)Wang, Kordi, Mishra, Liu, Smith, Khashabi, and Hajishirzi]{wang2023self}
Yizhong Wang, Yeganeh Kordi, Swaroop Mishra, Alisa Liu, Noah~A Smith, Daniel Khashabi, and Hannaneh Hajishirzi.
\newblock Self-instruct: Aligning language models with self-generated instructions.
\newblock In \emph{Proceedings of the 61st Annual Meeting of the Association for Computational Linguistics (Volume 1: Long Papers)}, pp.\  13484--13508, 2023.

\bibitem[Wei et~al.(2023)Wei, Wang, Liu, Ding, and Zhang]{wei2023magicoder}
Yuxiang Wei, Zhe Wang, Jiawei Liu, Yifeng Ding, and Lingming Zhang.
\newblock Magicoder: Source code is all you need.
\newblock \emph{arXiv preprint arXiv:2312.02120}, 2023.

\bibitem[Yang et~al.(2024)Yang, Klein, Celikyilmaz, Peng, and Tian]{yang2024rlcd}
Kevin Yang, Dan Klein, Asli Celikyilmaz, Nanyun Peng, and Yuandong Tian.
\newblock {RLCD}: Reinforcement learning from contrastive distillation for {LM} alignment.
\newblock In \emph{The Twelfth International Conference on Learning Representations}, 2024.
\newblock URL \url{https://openreview.net/forum?id=v3XXtxWKi6}.

\bibitem[Yuan et~al.(2024)Yuan, Pang, Cho, Sukhbaatar, Xu, and Weston]{yuan2024self}
Weizhe Yuan, Richard~Yuanzhe Pang, Kyunghyun Cho, Sainbayar Sukhbaatar, Jing Xu, and Jason Weston.
\newblock Self-rewarding language models.
\newblock \emph{arXiv preprint arXiv:2401.10020}, 2024.

\bibitem[Zhang et~al.(2024{\natexlab{a}})Zhang, Liu, Cherry, and Firat]{zhang2024when}
Biao Zhang, Zhongtao Liu, Colin Cherry, and Orhan Firat.
\newblock When scaling meets {LLM} finetuning: The effect of data, model and finetuning method.
\newblock In \emph{The Twelfth International Conference on Learning Representations}, 2024{\natexlab{a}}.
\newblock URL \url{https://openreview.net/forum?id=5HCnKDeTws}.

\bibitem[Zhang et~al.(2024{\natexlab{b}})Zhang, Qiao, Yang, and Wei]{zhang2024regurgitative}
Jinghui Zhang, Dandan Qiao, Mochen Yang, and Qiang Wei.
\newblock Regurgitative training: The value of real data in training large language models.
\newblock \emph{arXiv preprint arXiv:2407.12835}, 2024{\natexlab{b}}.

\bibitem[Zheng et~al.(2024)Zheng, Zhang, Shen, Liu, Lin, Fu, Chen, and Yue]{zheng2024opencodeinterpreter}
Tianyu Zheng, Ge~Zhang, Tianhao Shen, Xueling Liu, Bill~Yuchen Lin, Jie Fu, Wenhu Chen, and Xiang Yue.
\newblock Opencodeinterpreter: Integrating code generation with execution and refinement.
\newblock \emph{arXiv preprint arXiv:2402.14658}, 2024.

\end{thebibliography}

\clearpage

\clearpage

\appendix 
\section{More Works on Synthesized Data}\label{sec:related_append}

\subsection{Taxonomy for Synthesized Data}\label{sec:taxonomy}

\newcommand{\typeone}{\textcolor{cyan}{\ding{108}}}
\newcommand{\typetwo}{\textcolor{magenta}{\ding{115}}}
\newcommand{\typethree}{\textcolor{orange}{\ding{110}}}
\newcommand{\typefour}{\textcolor{violet}{\ding{70}}}

Contrary to the phenomenon of model collapse, synthesized data has been shown to improve performance in numerous empirical studies. We now provide a taxonomy outlining when and how synthesized data is beneficial. Specifically, we identify four key components: \textit{prompt engineering} \typeone, \textit{knowledge from advanced models} \typetwo, \textit{distributional shift and targeted curation} \typethree, and \textit{external verifiers} \typefour. Most empirical studies can be categorized based on one or more of these components. We use \typeone \typetwo \typethree and \typefour  to denote the components each reference leverages.

\textbf{Code and Math.} \cite{haluptzok2022language} \typefour generate synthesized data for codes and use a verifier to filter and show that the model can "self-improve" with its own synthesized data. \cite{gunasekar2023textbooks} \typeone \typethree filter high-quality data from the web and prompt GPT-3.5 with a specially curated prompt set covering both quality and diversity. \cite{wei2023magicoder} \typeone leverage a diverse and large set of open-source code snippets to curate code instructions as prompts with good coverage and high quality. \cite{zheng2024opencodeinterpreter, trinh2024solving} \typefour leverage a symbolic deduction engine as a verifier to test the correctness of each branch for solving Olympic geometry problems.

\textbf{Alignment.} During standard fine-tunings, synthesized data is often generated by a stronger model like GPT-4 \citep{peng2023instruction} \typetwo. \cite{wang2023self} \typeone \typefour use a good set of prompts and inputs with a heuristic verifier to filter out low-quality ones and maintain high diversity. \cite{bai2022constitutional} \typeone \typethree use the model itself to critique whether its own generation is harmful, given already harmful prompts with gold standards from humans. For alignment with reinforcement learning, \cite{ouyang2022training} \typeone \typefour use humans as verifiers to compare synthesized data generated by the current model with a good set of prompts. Some papers propose reinforcement learning with AI feedback (RLAIF) \citep{lee2023rlaif} \typeone \typethree that leverages another LLM as the verifier to match human verification. The verifier is a stronger model, instruct-tuned Palm2 L, while the network being trained is the Palm2 XS. However, \cite{yang2024rlcd} \typeone later found that using better prompts (self-improve) that direct harmful or harmless responses can surpass RLAIF. \cite{yuan2024self} \typeone achieve surprising results with iterative fine-tuning and generating good prompts with in-context learning.

\textbf{Knowledge distillation.} Most papers in the knowledge distillation area involve using a better model to distill for general performance or specific tasks, with \typeone, \typetwo, and \typethree involved from case to case. One example is the tiny story cases \citep{eldan2023tinystories} \typeone \typetwo, where GPT-4 is prompted to generate stories for four-year-olds that are used to train GPT-Neo with good performance.

\textbf{Image Domain.} \cite{kirillov2023segment} and \cite{li2022blip} \typethree use a distributional shift from high-quality to low-quality data to label and curate a vast amount of unlabeled data. Specifically, \cite{li2022blip} also trains a verifier to filters high-quality data. \citep{um2024dont} \typetwo \typethree specifically curate minority groups with a diffusion model to enhance performance. \cite{he2023is, dunlap2024diversify} \typetwo \typethree generate synthesized data that aids in classification tasks by tailoring the synthesized data to match the model's learning objectives. \cite{azizi2023synthetic, hemmat2023feedback} \typetwo \typethree employ guided curation (with supervision) to curate data from diffusion models. \cite{burg2023image} find that while synthesized data from a diffusion model helps improve downstream tasks, such as classification, using the pre-training data of the diffusion model alone gives even stronger performance.

\subsection{Knowledge Distillation with Soft Labels}

Related to synthesized data, there is a long history of using synthesized labels in image classifications. In the domains of self-distillation and knowledge distillation \citep{hinton2015distilling, furlanello2018born}, data with soft labels generated from the teacher model can significantly improve the performance of the student model. These soft labels convey additional insights—referred to as 'dark knowledge'—that have been theoretically linked to specific advantageous adaptations. These include implicit biases that mitigate overfitting \citep{mobahi2020self}, mimicry of early stopping \citep{dong2019distillation} for improved optimization under label noise \citep{das2023understanding}, and adjustments to accommodate specific data structures \citep{allen2022towards}. We only consider synthesized data with fixed labels as in the current practice of LLMs and diffusion models.

\subsection{Data Selection}\label{sec:related_append_3}

Comprehensive surveys on data selection for language models can be found in \cite{albalak2024survey}, along with theoretical studies on selection in high-dimensional settings \citep{sorscher2022beyond, kolossov2024towards}. Specifically, \cite{kolossov2024towards} also explore the use of surrogate models for producing labels during selection, followed by curation of the original labels. In our study, selection is applied to synthesized data where original labels are not available, resulting in distinct phenomena compared to these approaches on original data.

\section{Predicting the Eigenvalues} \label{sec:append-math}

We leverage the code base provided by \cite{charton2022linear} at \url{https://github.com/facebookresearch/LAWT} under the license CC BY-NC 4.0.

\textbf{Input and Tokenization.} Transformers are trained to predict the eigenvalues of $5\times 5$ symmetric real matrices. Model inputs are sequences of $25$ real entries, rounded to three significant digits, and tokenized as triplets of signs( \texttt{+} or \texttt{-}), mantissas (from \texttt{0} to \texttt{999}) and power of ten exponents (from \texttt{E-100} to \texttt{E100}). For instance, the $2 \times 2$ matrix, 
$$\begin{pmatrix} 2.3& 0.6035\\ 0.6035&  -3.141\end{pmatrix} $$
will be encoded as the sequence of $12$ tokens: \texttt{+ 23 E-1 + 604 E-3 + 604 E-3 - 314 E-2}. Model outputs are vectors of $5$ real eigenvalues, rounded to three significant digits, and tokenized as before, as triplets of sign, mantissa and exponent (the \texttt{P1000} encoding from \citep{charton2022linear}).

\textbf{Model and Optimization.} We train sequence-to-sequence transformers \citep{transformer17}, with $4$ layers in the encoder, and one in the decoder, $512$ dimensions and $8$ attention heads, to minimize a cross-entropy loss, using the Adam optimizer \citep{kingma2014adam}, with a fixed learning rate of $5\cdot 10^{-5}$, after an initial linear warm-up phase over the first 10,000 optimization steps. The model is trained for 400 epochs before overfitting.

\textbf{Evalution.} Model accuracies are measured on a held-out test set of examples not seen at training. 
Model predictions are evaluated by decoding the output sequence as a vector of $5$ real numbers $(p_1, p_2, p_3, p_4, p_5)$, and assessing that a prediction $\textbf p(p_1, p_2, p_3, p_4, p_5)$ of eigenvalues $\textbf{v}(v_1,v_2,v_3,v_4,v_5)$ is correct if the relative error in $L^1$ norms is below some tolerance $\tau$, i.e. if $$\sum_{i=1}^{5}|v_i - p_i| < \tau \sum_{i=1}^{5}|v_i|.$$
We use tolerances $\tau$ of $5, 2, 1$ and $0.5\%$.

\subsection{Finetuning Models with Synthesized Data}

In all previous experiments, the data generated from the generator (using beam or reject sampling) were used to train a new model. In this section, we consider using data generated from the generator to finetune models pre-trained on a small sample of ground truth data. We consider four cases: 

\begin{itemize}
\item Fine-tuning the generator (Model A).
\item Fine-tuning a model pre-trained to the same accuracy as the generator ($62\%$, Model B).
\item Fine-tuning a model pre-trained to higher accuracy ($93\%$, Model C).
\item Fine-tuning a model pre-trained to low accuracy ($4\%$, Model D).
\end{itemize}

\begin{table}[htb]
    \small
    \centering
    \begin{tabular}{l|ccccc}
        \toprule 
        & Model A (66\%) & Model B (62\%) & Model C (93\%) & Model D (4\%) & From scratch\\
        \midrule
         Rejection & 61.8 & 72.9 & 82.1 & 66.3 & 72.1 \\
         Beam 50  & 74.1 & 82.6 & 87.3 & 78.3 & 84.0 \\
         Beam 35 & 72.7 & 81.3 & 86.8 & 76.8 & 80.4\\
         Beam 25 & 71.3 & 79.8 & 84.4 & 73.3 & 79.9 \\
         Beam 10 & 67.5 & 75.1 & 83.5 & 68.0 & 73.9 \\
         Beam 5 & 64.9 & 70.8 & 80.1 & 65.6 & 69.1 \\
         Beam 1 & 61.6 & 62.1 & 75.6 & 55.8 & 60.5 \\
         \bottomrule
    \end{tabular}
    \caption{\small \textbf{Performance of models fine-tuned on 1M examples generated by the generator.} $\tau=2\%$}
    
   \label{tab:finetuning}
\end{table}

Table~\ref{tab:finetuning} compares accuracy of the four fine-tuning cases to that of a model trained from scratch. Fine-tuning only achiueves better performance when the pre-teained model achieved higher accuracy than model A. In all other cases, fine-tuning brings no improvement. Note that fine-tuning model A on its own generated data achieves the worst result, a clear case of model collapse.

\subsection{Computational Resources}

We leverage a V100 GPU with 32GB of memory for all experiments involving linear algebra. The training time ranges from 1 to 5 days, depending on the data size and the number of epochs. 

\section{Generalization of the Proxy} \label{sec:app_gen_proxy}
Recall that in our theoretical setting, we define all variables in a two-way classification problem, with $y$ as the ground truth label and $y'$ as the predicted synthesized label. We introduce an indicator variable $s$, defined as:
\[
s = 
\begin{cases} 
    1 & \text{if } y = y' \\
    0 & \text{if } y \neq y'
\end{cases}.
\]
We can rewrite all formulas in terms of $s$:
\[
p := 1 - \mathbb{E}[s],
\]
and
\[
\phi = \mathbb{P}(q=1 \mid s=1), \quad \psi = \mathbb{P}(q=1 \mid s=0).
\]

When moving beyond the classification setting, $s$ is no longer a binary 0/1 variable. In more general cases, $s$ represents a measure of similarity or distance between the original target $y$ and the synthesized target $y'$. Without loss of generality, we assume $s$ is normalized to lie within the range $[0, 1]$. For instance, in the news summarization experiment, $s$ corresponds to the ROUGE-1 score between the ground truth text $y$ and the synthesized text $y'$, and equals 1 if and only if $y = y'$.

With this generalized measure $s$, we can extend the previous definitions of $\phi$ and $\psi$ as follows:
\[
\phi = \frac{\mathbb{E}_s [qs]}{1 - p}, \quad \psi = \frac{\mathbb{E}_s [q(1 - s)]}{p}.
\]
When $s$ takes values in $\{0, 1\}$, these definitions are consistent with those in Equation \ref{equ:psiphi}.

\section{News Summarization}\label{sec:news_append}

We leverage the XLSUM dataset \citep{hasan-etal-2021-xl} at \url{https://huggingface.co/datasets/csebuetnlp/xlsum} under the license CC-BY-NC-SA 4.0. 

\textbf{Data preprocessing.} For each data in both training and test dataset, it consists of a news report and a summarization, denoted as (\textit{news}, \textit{summarization}). We write each data in the following form:
\begin{tcolorbox}[colback=blue!10, colframe=black, boxrule=1pt]
Article: \textit{news}. A summary of the article: \textit{summarization}. 
\end{tcolorbox}

\textbf{Fine-tuning and generating details.} Throughout all phases of evaluation and generation, we employ greedy decoding. Given that news summarization is a low-entropy task, greedy decoding is chosen to ensure quality generation. Consistent with common practice, fine-tuning is limited to a single epoch. Through out the experiments, all the finetuning is with full parameter tuning to better capture the scaling law as suggested in \cite{zhang2024when}. 

\textbf{Implementation details.} We leverage the official implementation in Huggingface \footnote{\url{https://github.com/huggingface/transformers/blob/main/examples/pytorch/language-modeling/run_clm.py}} for training, under the license Apache 2.0. Specifically, for training the generator, we start our training with the pre-trained Llama-2, and set the learning rate to 5e-5, the learning rate scheduler as `cosine', the number of epochs to 1, the total batch size to 32, the block size to 1024 and the others to the default value. For generating the synthesized data, we use greedy strategy to generate a summarization for each news in the training set. For training based on the selected synthesized data, we also start our training with the pre-trained Llama-2, and set the learning rate to 2e-5, the learning rate scheduler as `constant' and the others to the same. For evaluation, we first use greedy strategy to generate a summarization for each news in the test set, and then calculate the Rouge-1 score between the generated summarization and the corresponding ground truth, and finally report the average of the Rouge-1 scores of all test data. When calculating the perplexity, we only calculate the perplexity for the generated summary. When fine-tuning the Llama-3 model, we use the full XLSUM dataset to achieve good performance. The resulting model achieves a Rouge-1 score of 34.5.

\textbf{Computational Resources.} All experiments were conducted using a dedicated computational cluster equipped with 4 NVIDIA A800 GPUs, each with 80 GB of memory. Our training and inference processes are performed on the cluster.

\textbf{Estimated Time.} Training the whole dataset for an epoch takes about 6 hours. Generating the whole dataset takes about 1 day. During evaluation, we need to first generate and calculate the rouge score, which takes around 40 minutes for one checkpoint.

\section{A General Theory of Pruning with Verification}\label{app:theory}

In Section \ref{sec:theory_main_paper} we have presented a special case of our general theory, which we describe here in more generality and detail. While some of our exposition here overlaps with Section \ref{sec:theory_main_paper}, we prefer to leave it as a complete text that provides a stand-alone overview.

\subsection{Data Distribution} 
Consider a probability distribution $P$ over $\mathbb R^d \times \{0,1\}$ with the following high-dimensional property
\begin{condition}
\label{cond:main}
Given $N \le N(d)$ i.i.d. samples $(x_1,y_1),\ldots,(x_N,y_N)$ from $P$ with $N \le N(d)$, the following hold estimates w.p $1-o(1)$ uniformly on all $i,j \in [N]$, in the limit $d \to \infty$
\begin{align*}
    \|x_i\|^2 &\simeq 1\label{eq:unit-norm},\\
    x_i^\top x_j &\simeq \begin{cases}
        a,&\mbox{ if }y_i = y_j,\\
    b,&\mbox{ if }y_i \ne y_j
    \end{cases}
\end{align*}
where $b < a < 1$ are constants. For simplicity of presentation of our results, We will further assume that $b=-a$ or $b=0$.
\end{condition}
The above structural condition is inspired by an assumption in \cite{das2023understanding}.

For simplicity of exposition, we will only consider balanced distributions, meaning that
\begin{eqnarray*}
\mathbb P(y=1) = \mathbb P(y=0) = 1/2,\text{ for }(x,y) \sim P.
\end{eqnarray*}
    

\paragraph{Gaussian Mixture Example.}
As a first example, in the case of Gaussian mixtures where the features have conditional distribution given by
\begin{eqnarray}
x \mid y \sim N(\mu_y,\Sigma),\\
\end{eqnarray}
where $\mu_y = (2y-1)\mu$, for some $\mu \in \mathbb R^d$ and positive-definite matrix $\Sigma$ with $\mathbb E\,\|x\|^2 = \|\mu\|^2+\operatorname{tr}\Sigma = 1$, we may take
\begin{eqnarray}
a=\|\mu\|^2,\quad b= -a. 
\end{eqnarray}
Condition \ref{cond:main} then holds thanks to concentration, with $N(d) = e^{\Theta(d)}$.

\subsection{Training Data, Data Pruning, and Downstream Model}
Let $D_N=\{(x_1,y_1),\ldots,(x_N,y_N)\}$ be a dataset of $N$ i.i.d. pairs from the true distribution $P$ and let $D_N'=\{(x_1,y_1'),\ldots,(x_N,y_N')\}$ a version of the dataset (also i.i.d.) with labels $y'_i$ instead of $y_i$. For example, this could be labels generated by an AI trying to reproduce real-world data. $D'_N$ is the data on which the downstream model is trained.  

We will consider a family of models given by
\begin{eqnarray*}
\mathbb P (y =1\mid x,w) = \hat y := \sigma(x^\top w) \in (0,1),
\end{eqnarray*}
parametrized by a vector of weights $w \in \mathbb R^d$. Here, $\sigma$ is the sigmoid function defined by
\begin{eqnarray} 
\sigma(z):=\frac{1}{1+e^{-z}}.
\end{eqnarray}
For the loss function, we use binary cross-entropy (BCE), defined by
\begin{eqnarray}
    \ell(\hat y,y) = -y\log\hat y - (1-y)\log (1-\hat y).
\end{eqnarray}
Let $\widehat w_N$ be obtained via logistic regression fitted on $D'_N$ with ridge regularization parameter $\lambda>0$. Thus, $\hat w$ is the unique\footnote{Unicity is due to strong convexity of objective function.} minimizer of the following objective function:
\begin{eqnarray}
L(w) := \frac{1}{N}\sum_{i=1}^N q_i\ell(\sigma(x_i^\top w), y_i') + \frac{\lambda}{2}\|w\|^2.\nonumber
\end{eqnarray}
Here $q_i$ is a bit which indicates whether the $i$th training example has survived pruning. The numbers $q=(q_1,\ldots,q_N)$ is called a \emph{pruning strategy}. The corresponding downstream classifier is $\widehat f_N = f_{\widehat w_N}$, where the notation $f_w$ refers to the linear classifier induced by a weights vector $w \in \mathbb R^d$, i.e
\begin{eqnarray}
 f_w(x) := \begin{cases}
1,&\mbox{ if }x^\top w > 0,\\
0,&\mbox{ otherwise.}
\end{cases}
\end{eqnarray}
The test accuracy of the downstream model $\widehat f_N$ is defined by
$$
acc(\widehat f_N) := \mathbb P(\widehat f_N(x) = f_{Bayes}(x)),\text{ for a random test point }(x,y) \sim P,
$$
where $f_{Bayes}(z) := \mathbb E[y |x=z]$ is the Bayes-optimal classifier. In particular, note that $acc(f_{Bayes})=100\%$ by construction.

This quantity will be the main object of our analysis, and we will be interested in how it depends on the corruption level $p$ and the choice of pruning strategy $q$, in the infinite-sample limit $N \to \infty$.

For later reference, we also define an empirical version, namely the accuracy of $\widehat f_N$ evaluated on the clean dataset $D_N$, namely
\begin{eqnarray}
        \widehat{acc}(\widehat f_N) := \frac{1}{|M|}|\{i \in M \mid \widehat f_N(x_i) = y_i\}|,
\end{eqnarray}
where $M := \{i \in [N] \mid q_i = 1\}$ collects the indices of training samples which survive pruning by $q$.

\subsection{A Class of Parametrized Pruning Strategies}
Given hyper-parameters $\phi_0,\phi_1,\psi_{01},\psi_{10} \in [0,1]$, we consider a broad class of parametrized pruning strategies with the following property. For any class labels $k,\ell \in \{0,1\}$, the random variables $(z_{ik\ell})_{i \in [N]}$ defined by $z_{ik\ell}=1[y_i=k,y_i'=\ell,q_i=1]$ are i.i.d. with Bernoulli distribution $Bern(p_{k\ell})$, where
\begin{eqnarray}
\label{eq:pkl}
    \begin{split}
p_{k\ell} &= \mathbb P(q_i = 1, y'_i=\ell,y_i=k)\\
&= \mathbb P(q_i=1 \mid y'_i=\ell,y_i=k)\mathbb P(y_i'=\ell \mid y_i=k)\mathbb P(y_i=k)\\
&=
\begin{cases}
    \phi_k(1-p)/2,&\mbox{ if }k=\ell,\\
    \psi_{k\ell}p/2,&\mbox{ else.}
\end{cases}    
    \end{split}
\end{eqnarray}
and the numbers $p$, $\phi_k$ and $\psi_{k\ell}$ are defined by
\begin{eqnarray}
p := \mathbb P(y'_i\ne y_i),\quad \phi_k = \mathbb P(q_i = 1 \mid y'_i=k,y_i=k),\quad \psi_{k\ell}= \mathbb P(q_i = 1 \mid y'_i=\ell,y_i=k).
\end{eqnarray}
Consequently, if $N_{k\ell}$ is the number of training examples that have true label $k$, fake label $\ell$, and survive pruning, then
\begin{eqnarray}
\label{eq:Nkl}
N_{k\ell} := \sum_{i=1}^N z_{ik\ell}
\end{eqnarray}
which is has binomial distribution $Bin(N,p_{k\ell})$. As mentioned in the main text, for simplicity of exposition we considered the following simplifying assumption:
\begin{eqnarray}
    \phi_1 = \phi_0 = \phi,\quad \psi_{01}=\psi_{10} = \psi.
\end{eqnarray}
Such a pruning strategy will be referred to as a verification-pruning strategy with parameter $(\phi,\psi)$. 
 \begin{remark}
    Note that the parametrization $(\phi,\psi)$ and $(p_{00},p_{11})$ describe the same verification-pruning strategy via the following bijective transformation.
    \begin{align}
        p_{00} &= p_{11} = (1-p)\phi/2,\quad p_{01}=p_{10} = 
 p \psi /2.
    \end{align}
\end{remark}

\subsection{Examples}
Let us present some notable examples of verification-pruning strategies. 

\paragraph{No Selection.} The case $(\phi,\psi) = (1,1)$ corresponds to no selection, i.e the entire training dataset is used. 
     \paragraph{Oracle Selection.} The case $(\phi,\psi) = (1,0)$. The selection strategy only keeps indices corresponding to examples in the dataset which have correct label (all corrupted labels discarded). 

\paragraph{Supervised (Margin-Based) Selection.} Let $w_{prune} \in \mathbb R^d$, and consider the pruning strategy defined by
    \begin{eqnarray}
    q_i = 1[y'_i(x_i^\top w_{prune})>0].\nonumber
    \end{eqnarray}
    This pruning strategy simplify filters out all examples on which it disagrees on the assigned label.

\subsection{Performance Bounds for Models Trained with Verification Selection}
    \begin{figure}[htb]
    \centering
        \includegraphics[width=.45\textwidth]{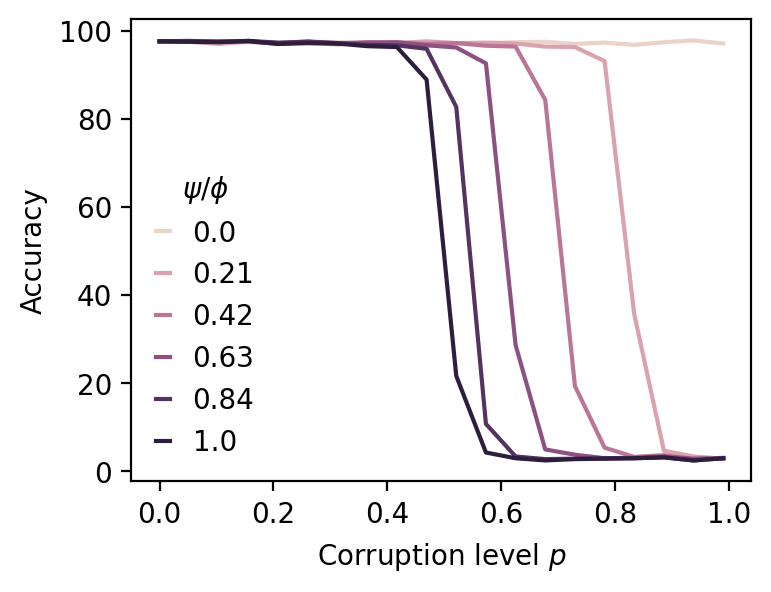}
        \includegraphics[width=.43\textwidth]{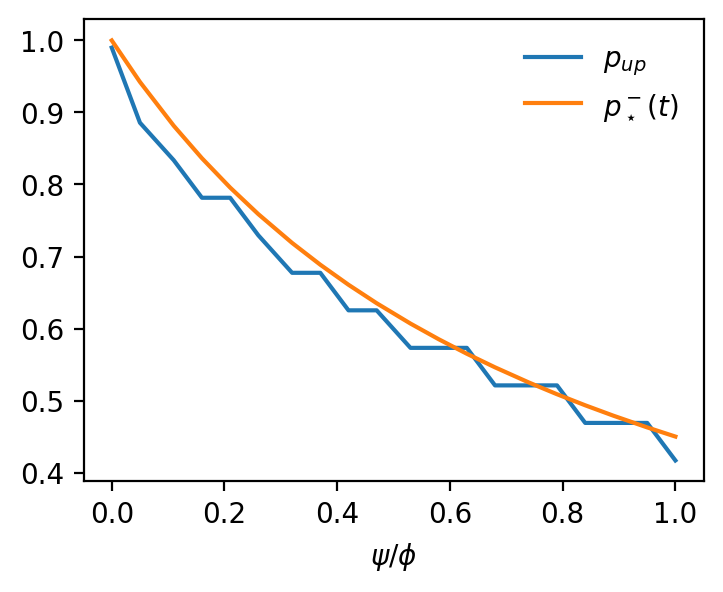}
   \caption{\textbf{Empirical Confirmation of Theorem \ref{thm:main}.} Comparing the breakdown points of different models. Here, the task is classifying a Gaussian mixture, with infinite training samples from datasets generated from a model with classification error rate $p$ (x-axis). Notice the sharp phrase-transitions where the model suddenly switches from perfect accuracy to worse-than-chance, a phenomenon predicted by Theorem \ref{thm:main}. 
    \textbf{Left.} Performance of verification-pruning strategies with different values of the  hyper-parameters $(\phi,\psi)$. Recall that the case $\psi/\phi=1$ corresponds to no pruning, while $\psi/\phi=0$ corresponds to oracle selection. \textbf{Right.} Comparing $p_{up}$, approximated with $\sup\{p \mid acc(\widehat f_N) \ge 90\%\}$ (computed empirically), against the analytic estimate $p^-_\star(t)$ given in Theorem \ref{thm:main} (for $t=0.1$). Again, the results are in excellent agreement with the predictions of the theorem. 
    }
    \label{fig:curves}
\end{figure}
The following is one of our main results (proved in Appendix \ref{sec:proof}).
\begin{theorem}
\label{thm:main}
Suppose Condition \ref{cond:main} is in order. Fix $\phi,\psi,t \in (0,1)$ and define $p^\pm_\star(t) \in (0,1)$ by
\begin{eqnarray}
    p^-_\star(t) := \frac{1 - t}{1 - t + (1 + t)\psi/\phi},\quad p^+_\star:=\frac{1 + t}{1 + t + (1 - t)\psi/\phi}
    \label{eq:threshold}
\end{eqnarray}
If $p < p^-_\star(t)$, then the limit $N \to \infty$ it holds w.p $1-o(1)$  that the $acc(\widehat f_N) = 100\%$ for a downstream model $\widehat f_N$ trained on data from  a generator with error rate $p$ pruned with a verification-pruning strategy  with parameters $(\phi,\psi)$.

On the other hand, if $p > p^+_\star$, then in the limit $N \to \infty$ it holds w.p $1-o(1)$  that the $acc(\widehat f_N) = 0\%$ for a downstream model $\widehat f_N$.

Thus, there is a sharp phase-transition around the corruption level $p_\star := 1/(1+\psi/\phi)$: as $p$ is increased past level $p_\star$, the downstream model $\widehat f_N$ abruptly switches from being perfectly accurate, to perfectly inaccurate!
\end{theorem}
See Figure \ref{fig:curves} for an empirical illustration of the theorem.

The thresholds $p^\pm_\star(t)$ appearing in the above theorem are proxies for the so-called \emph{breakdown points} $p_{up} \ge p_{down}$ defined by
\begin{align}
    p_{up} &= \inf\left\{p \in [0,1] \,\big|\, acc(\widehat f_N) \overset{a.s}\to 0\%\text{ in the limit }N \to \infty\right\},\\
        p_{down} &= \sup\left\{p \in [0,1] \,\big|\, acc(\widehat f_N) \overset{a.s}\to 100\%\text{ in the limit }N \to \infty\right\}.
\end{align}
Theorem \ref{thm:main} implies $p_{down} \ge p^-_\star(t)$ and $p_{up} \le p^+_\star(t)$ for all $t \in (0,1)$. Consequently,
\begin{corollary}
    Under the hypotheses of Theorem \ref{thm:main}, it holds that $p_{up} = p_{down}$.
\end{corollary}

\subsection{Some Consequences of Theorem \ref{thm:main}}\label{app:consequences}
We now present some illustrious applications of Theorem \ref{thm:main}. These examples are empirically confirmed in Figure \ref{fig:curves}.

\paragraph{No Selection.}
Here, we have $\psi/\phi=1$ and so the downstream model achieves $100\%$ accuracy for all values of corruption parameter $p$ up to the proxy  breakdown point predicted by Theorem \ref{thm:main} is then $p^-_\star = 1/2-t/2$.

\paragraph{Oracle Selection.} For this scenario, $\psi/\phi=0$ and  so Theorem \ref{thm:main} predicts that the downstream model $\widehat f_N$ achieves $100\%$ accuracy for all values of corruption parameter $p$ up to the breakdown point $p^-_\star = 1$. This is perhaps not so surprising in hindsight. The point is that even for moderately large values of $\psi/\phi$, the proxy breakdown point $p^-_\star$ given in \eqref{eq:threshold} can still be quite close to $1$.

\paragraph{Self-supervised (Margin-Based) Selection.}
Consider Gaussian mixture data with means $\pm \mu$, and consider a margin-based pruning strategy in Equation (\ref{eq:pruner}). It is clear that $\phi$ and $\psi$ only depend on all the 3 angles between the set of vectors $\{w_*,w_{gen},w_{prune}\}$, with $w_*=\mu$. 

\paragraph{Supervised Selection.}
Consider isotropic Gaussian mixture data with means $\pm \mu$, and a pruning strategy as in Eq.~(\ref{eq:pruner}). The parameters $(\phi,\psi)$ only depend on the angles $\theta_{gen},\theta_{prune},\theta \in [0,\pi]$ given by 
\begin{eqnarray}
\begin{split}
\theta_{gen}&:=\angle (w_{gen},\mu),\, \theta_{prune} := \angle (w_{prune},\mu),\\
\theta &:= \angle (w_{prune},w_{gen}).
\end{split}
\end{eqnarray}

Conditioned on \(y = 1\) and using \(\mathbb{E}\|x\|^2 = \|\mu\|_2^2 + \operatorname{tr}\Sigma = 1\) (and \(\|\mu\|_2 <1\)), we get \(x \sim \mathcal{N}(\mu, \frac{1}{(1 - \|\mu\|_2^2) d} \mathbf{I}_d)\). We can rewrite \(x\) as:
\begin{equation}\label{eq:eta}
x = \mu + \frac{1}{(1 - \|\mu\|_2^2)d} \eta,
\end{equation}
where \(\eta \sim \mathcal{N}(0, \mathbf{I}_d)\). For simplicity of calculation, let us further assume here that \(\|w_{\text{gen}}\|_2^2 = \|w_{\text{prune}}\|_2^2 = 1\).

\[
\begin{aligned}
    \phi_1 &= \mathbb{P}(x^\top w_{\text{prune}} > 0 \mid y = 1, y' = 1) \\
    &= \frac{\mathbb{P}(x^\top w_{\text{prune}} > 0, y = 1, x^\top w_{\text{gen}} > 0)}{\mathbb{P}(y = 1, x^\top w_{\text{gen}} > 0)} \\
    &= \frac{\mathbb{P}(\eta^\top w_{\text{prune}} > -\|\mu\|_2 d (1 - \|\mu\|_2^2) \cos \theta_{\text{prune}}, \eta^\top w_{\text{gen}} > -\|\mu\|_2 d (1 - \|\mu\|_2^2) \cos \theta_{\text{gen}})}{\mathbb{P}(\eta^\top w_{\text{gen}} > -\|\mu\|_2 d (1 - \|\mu\|_2^2) \cos \theta_{\text{gen}})}.
\end{aligned}
\]
In the second step, we use the definition of conditional probability and use \eqref{eq:eta} from line 2 to line 3.
The random variables \(\eta^\top w_{\text{gen}}\) and \(\eta^\top w_{\text{prune}}\) are jointly Gaussian with:
\[
\begin{pmatrix}
\eta^\top w_{\text{gen}} \\
\eta^\top w_{\text{prune}}
\end{pmatrix}
\sim \mathcal{N} \left( \mathbf{0}, \begin{pmatrix}
1 & \cos \theta \\
\cos \theta & 1
\end{pmatrix} \right).
\]

Let \(\Phi\) be the CDF of the standard normal distribution and \(\Phi_2\)  the CDF of the bivariate normal distribution, where \(\Phi_2(x, y; \rho)\) is defined as:
\[
\Phi_2(x, y; \rho) := \mathbb{P}(X \le x, Y \le y)
\]
for \((X, Y) \sim \mathcal{N} \left( \mathbf{0}, \begin{pmatrix}
1 & \rho \\
\rho & 1
\end{pmatrix} \right)\). Denote \(c_1 = -\|\mu\|_2 d (1 - \|\mu\|_2^2) \cos \theta_{\text{prune}}\) and \(c_2 = -\|\mu\|_2 d (1 - \|\mu\|_2^2) \cos \theta_{\text{gen}}\). We have:
\[
\begin{aligned}
    \phi_1 &= \frac{1 - \Phi(c_1) - \Phi(c_2) + \Phi_2(c_1, c_2; \cos \theta)}{1 - \Phi(c_2)}.
\end{aligned}
\]

All the distributions are symmetric, and we have $\phi_0=\phi_1=\phi$. In the same spirit, $\psi = \psi_{10} = \psi_{01}$, with 
\[
\begin{aligned}
    \psi= \frac{\Phi(c_1) + \Phi(c_2) - \Phi_2(c_1, c_2; \cos \theta)}{ \Phi(c_2)}.
\end{aligned}
\]

\subsection{Sketch of Proof of Theorem \ref{thm:main}}\label{app:sketch}
The proof is based on the following representation (refer to Proposition \ref{prop:acc}) of the accuracy of the downstream classifier $\widehat f_N$ evaluated on the the clean training dataset $D_N$, namely
\begin{eqnarray}
\widehat{acc}(\widehat f_N) = \frac{N_{11}1_{\overline A < 1/2} + N_{00} 1_{\overline D < 1/2} + N_{10}1_{\overline B > 1/2} + N_{01} 1_{\overline C > 1/2}}{N_{11} + N_{00} + N_{10} + N_{01}},
\end{eqnarray}
for some random some random variables $\overline A,\overline B,\overline C,\overline D \in (0,1)$ which depend on the $N_{k\ell}$'s defined in \eqref{eq:Nkl}.

\begin{remark}
We only compute the accuracy $\widehat{acc}(\widehat f_N)$ of the downstream model $\widehat f_N$ evaluated on the clean training dataset $D_N$. By classical results in learning theory \citep{mcAllester2003,BenDavidUnderstanding,kakade2008complexity}, we know that the gap to the population version (test accuracy) $acc(\widehat f_N)$ shrinks to zero at rate $O(1/\sqrt N)$, and so since the claim in Theorem \ref{thm:main} is made only in the limit $N \to \infty$, we are good.
\end{remark}

Next, in Proposition \ref{prop:symmetric-case} and Proposition  \ref{prop:skewed-case}, necessary and sufficient conditions are established to ensure $\overline A,\overline D < 1/2$ and $\overline B,\overline C > 1/2$, and therefore $\widehat{acc}(\widehat f_N) = 100\%$. These conditions are given explicitly in terms of the $N_{k\ell}$'s. Finally, in Proposition \ref{prop:concentration}, concentration of measure is used to control the $N_{k\ell}$'s, and present the aforementioned conditions in terms of the $p_{k\ell}$'s defined in \eqref{eq:pkl}, and therefore in terms of $p$, $\phi$, and $\psi$ alone, giving condition \eqref{eq:threshold}.
\label{prop:acc-informal}

\section{Proof of Theorem \ref{thm:main}}\label{sec:proof}
Our analysis is based on non-trivial extensions of arguments by Das and Sanghavi (2023). Viz,  
\begin{itemize}
    \item We allow for a selection mechanism (aforementioned work does study selection, just self-distillation), and
    \item We use a careful asymptotic analysis to avoid solving certain complicated fixed-point equations defining the weights vector $\widehat w_N$ of the downstream model $\widehat f_N$.
\end{itemize}

\subsection{Preliminary Computations}
For later use, given a pruning strategy $q$, define the following objects
\begin{align}
I_k &:= \{j \in [N] \mid y_j = k\},\\
I'_\ell &:= \{j \in [n] \mid y_j' = \ell\},\\
M &:= \{i \in [N] \mid q_i = 1\},\\
N_{k\ell} &:= \sum_{i \in I_k \cap I'_\ell} q_i = |I_k \cap I'_\ell \cap M|,\\
R&:=1-a>0.
\end{align}
Thus, $N_{k\ell}$ is the number of training examples that have true label $k$, fake label $\ell$, and survive pruning. The following result will be crucial in the sequel.
\begin{proposition}
We have the representation $\widehat w = \sum_{i \in M}\alpha_i x_i$, where
\begin{align}
\alpha_i &= \begin{cases}
A,&\mbox{ if }i \in I_1 \cap I_1' \cap M,\\
-B,&\mbox{ if }i \in I_1 \cap I_0' \cap M,\\
C,&\mbox{ if }i \in I_0 \cap I_1' \cap M,\\
-D,&\mbox{ if }i \in I_0 \cap I_0' \cap M,
\end{cases}
\end{align}
and $A,B, C, D \ge 0$ solve the following system of equations
\begin{eqnarray}
\label{eq:master}
\begin{split}
\gamma  A &= \sigma(-(aN_{11}A-aN_{10}B+b N_{01}C-bN_{00}D)-RA),\\
\gamma  B &= \sigma(aN_{11}A-aN_{10}B+b N_{01}C-bN_{00}D  - RB),\\
\gamma  C &= \sigma(-(bN_{11}A-bN_{10}B+a N_{01}C-aN_{00}D)-RC),\\
\gamma  D &= \sigma(bN_{11}A-bN_{10}B+a N_{01}C-aN_{00}D - RD).
\end{split}
\end{eqnarray}
\end{proposition}
\begin{proof}
The following result is inspired by \citep{das2023understanding} and the proof is similar. Observe that
KKT conditions $\nabla L(w) = 0$ give $\sum_{i=1}^Nq_i(\hat y_i-y_i')x_i + \gamma w = 0$, i.e
\begin{eqnarray}
w = \sum_{i=1}^N q_i\alpha_ix_i,\text{ with }\alpha_i := \frac{y_i'-\hat y_i}{\gamma},\, \hat y_i := \sigma(v_i),\, v_i = x_i^\top w.
\end{eqnarray}
One then computes 
\begin{eqnarray}
\begin{split}
v_i &= x_i^\top w = \sum_{j=1}^N q_i\alpha_i x_i^\top x_j = q_i\alpha_i + \begin{cases}
a(s-q_i\alpha_i) + bt,&\mbox{ if }i \in I_1,\\
a(t-q_i\alpha_i) + bs,&\mbox{ if }i \in I_0,
\end{cases}\\
&= \begin{cases}
as+bt+Rq_i\alpha_i,&\mbox{ if }i \in I_1,\\
bs+at + Rq_i\alpha_i,&\mbox{ if }i \in I_0,
\end{cases}
\end{split}
\end{eqnarray}
where $s \ge 0$ and $t \ge 0$ are given by
\begin{align}
s &:= \sum_{j \in I_1} q_j\alpha_j,\quad
t :=\sum_{i \in I_0}q_j\alpha_j.
\end{align}
We deduce that for any $i \in M$,
\begin{eqnarray}
\begin{split}
\gamma \alpha_i = y_i'-\sigma(v_i) = \begin{cases}
1-\sigma(as+bt+Rq_i\alpha_i),&\mbox{ if }i \in I_1 \cap I_1',\\
-\sigma(as+bt+Rq_i\alpha_i),&\mbox{ if }i \in I_1 \cap I_0',\\
1-\sigma(bs+at+Rq_i\alpha_i),&\mbox{ if }i \in I_0 \cap I_1',\\
-\sigma(bs+at+Rq_i\alpha_i),&\mbox{ if }i \in I_0 \cap I_0'.
\end{cases}
\end{split}
\end{eqnarray}
Due to monotonicity of $\sigma$, we deduce the existence of $A,B, C, D \ge 0$ such that
\begin{align}
\alpha_i &= \begin{cases}
A,&\mbox{ if }i \in I_1 \cap I_1' \cap M,\\
-B,&\mbox{ if }i \in I_1 \cap I_0' \cap M,\\
C,&\mbox{ if }i \in I_0 \cap I_1' \cap M,\\
-D,&\mbox{ if }i \in I_0 \cap I_0' \cap M.
\end{cases}\\
\hat y_i &= y_i'-\gamma\alpha_i = \begin{cases}
    1-\gamma A,&\mbox{ if }i \in I_1 \cap I_1' \cap M,\\
    \gamma B,&\mbox{ if }i \in I_1 \cap I_0' \cap M,\\
    1-\gamma C,&\mbox{ if }i \in I_0 \cap I_1' \cap M,\\
    \gamma D,&\mbox{ if }i \in I_0 \cap I_0' \cap M.
\end{cases}
\end{align}
Furthermore, these scalars must verify
\begin{eqnarray}
\label{eq:premaster}
\begin{split}
\gamma A &= 1-\sigma(as+bt+RA) = \sigma(-(as+bt)-RA),\\
\gamma B &= \sigma(as + bt  - RB),\\
\gamma C &= 1-\sigma(bs+at+RC) = \sigma(-(bs+at)-RC),\\
\gamma D &= \sigma(bs+at - RD).
\end{split}
\end{eqnarray}
Finally, observe that,
\begin{align}
 s&=N_{11}A-N_{10}B,\quad t = N_{01}C-N_{00}D,
 \end{align}
from which we get
\begin{align*}
as + bt &= a (N_{11} A-N_{10} B) + b(N_{01}C -N_{00} D)\\
&= aN_{11} A - a N_{10} B + bN_{01}  C - b N_{00} D,\\
bs + at &= b (N_{11} A-N_{10} B) + a(N_{01}C -N_{00} D)\\
&= bN_{11} A - b N_{10} B + aN_{01} C - a N_{00} D.
\end{align*}
Plugging this into  \eqref{eq:premaster} gives \eqref{eq:master}.
\end{proof}

\subsection{Analytic Formula for Accuracy Evaluated Clean Training Data}
One computes the accuracy $\widehat{acc}(\widehat f_N)$ of the downstream model evaluated on the clean training dataset $D_N$ as
$$
\widehat{acc}(\widehat f_N) = \frac{1}{|M|}\left(\left|\{i \in M \mid y_i=1 \land \hat y_i > 1/2 \text{ OR }y_i=0\land \hat y_i < 1/2\}\right|\right).
$$
We can rewrite this as follows
\begin{eqnarray}
\begin{split}
|M| \cdot \widehat{acc}(\widehat f_N) &= \left|\{i \in M \mid y_i=1 \land \hat y_i > 1/2 \text{ OR }y_i=0\land \hat y_i < 1/2\}\right|\\
&=\left|\{i \in I_1 \cap M \mid \hat y_i > 1/2\}\right| + \left|\{i \in I_0 \cap M \mid \hat y_i < 1/2\}\right|\\
&=\sum_{i \in I_1 \cap M}1_{\hat y_i > 1/2} + \sum_{i \in I_0 \cap M}1_{\hat y_i < 1/2}\\
&= |I_1 \cap I_1' \cap M|1_{\gamma A < 1/2} + |I_1 \cap I_0' \cap M| 1_{\gamma B > 1/2}\\
&\quad + |I_0 \cap I_1' \cap M|1_{\gamma C > 1/2} + |I_0 \cap I_0' \cap M| 1_{\gamma D < 1/2}\\
&= N_{11}1_{\gamma A < 1/2} + N_{00} 1_{\gamma D < 1/2} + N_{10}1_{\gamma B > 1/2} + N_{01} A_{\gamma C > 1/2}.
\end{split}
\end{eqnarray}
On the other hand, it is clear that the size of the mask is $|M| = \sum_{k,\ell} N_{k\ell}$. Putting things together gives the following result which shall be crucial in the sequel.
\begin{proposition}
For any $\phi,\psi \in [0,1]$, there is a solution $(A,B,C,D)$ of the system of equations \eqref{eq:master} such that 
\begin{eqnarray}
\widehat{acc}(\widehat f_N) = \frac{N_{11}1_{\overline A < 1/2} + N_{00} 1_{\overline D < 1/2} + N_{10}1_{\overline B > 1/2} + N_{01} 1_{\overline C > 1/2}}{N_{11} + N_{00} + N_{10} + N_{01}},
\end{eqnarray}
where $\overline A := \gamma A$, $\overline B = \gamma B$, $\overline C = \gamma C$, and $\overline D = \gamma D$ as usual.
\label{prop:acc}
\end{proposition}
Thus, to attain $100\%$ accuracy, it suffices to have $\overline A,\overline D < 1/2$ and $\overline B,\overline C > 1/2$. The proof of Theorem \ref{thm:main} will be all about establishing sufficient conditions which ensure these inequalities.

\subsection{Sufficient Conditions for Perfect Accuracy}
Note that since $\gamma = N \lambda$ with $\lambda>0$ fixed and $N \to \infty$, we have $\gamma \to \infty$ and system of equations \eqref{eq:master} simplify to\footnote{These simplifications are made possible by the \emph{Mean Value Theorem}.} 
\begin{eqnarray}
    \begin{split}
\overline B &= \sigma((aN_{11}\overline A-aN_{10}\overline B+b N_{01}\overline C-bN_{00}\overline D)/\gamma),\\
\overline A &= \sigma(-(aN_{11}\overline A-aN_{10}\overline B+b N_{01}\overline C-bN_{00}\overline D)/\gamma)=1-\overline B,\\
\overline D &= \sigma((bN_{11}\overline A-bN_{10}\overline B+a N_{01}\overline C-aN_{00}\overline D)/\gamma),\\
\overline C &= \sigma(-(bN_{11}\overline A-bN_{10}\overline B+a N_{01}\overline C-aN_{00}\overline D)/\gamma)=1-\overline D,
\end{split}
\end{eqnarray}
where $\overline A := \gamma A$, $\overline B = \gamma B$, $\overline C = \gamma C$, $\overline D = \gamma D$ as usual, and we have used the elementary property that $\sigma(-z) = 1-\sigma(z)$. Eliminating $\overline A$ and $\overline C$, the above equations further collapse to
\begin{eqnarray}
    \begin{split}\overline B &= \sigma((aN_{11}(1-\overline B)-aN_{10}\overline B+b N_{01}(1-\overline D)-bN_{00}\overline D)/\gamma),\\
&= \sigma((aN_{11}+bN_{01}-a(N_{11}+N_{10})\overline B -b(N_{01} + N_{00})\overline D)/\gamma),\\
\overline D &= \sigma((bN_{11}(1-\overline B)-bN_{10}\overline B+a N_{01}(1-\overline D)-aN_{00}\overline D)/\gamma)\\
&= \sigma((bN_{11}+aN_{01}-b(N_{10}+N_{10})\overline B - a(N_{01}+N_{00})\overline D)/\gamma).
\end{split}
\label{eq:master-simplified}
\end{eqnarray}
Two special cases are tractable.

\paragraph{The Symmetric Case: $b=-a$.}
We have $\overline D = 1-\overline B$, and thus the equations become
\begin{eqnarray}
    \begin{split}
\overline D &= \overline A,\quad \overline C=\overline B,\quad \overline D=1-\overline B,\\
\overline B &=  \sigma((a(N_{11}+N_{00})-a(N_{11} + N_{10} + N_{01} + N_{00})\overline B)/\gamma).
\end{split}
\end{eqnarray}
If $\overline B \le 1/2$, then we must have $\overline B \ge (N_{11}+N_{00})/(N_{11} + N_{10} + N_{01} + N_{00})$, which is impossible if we impose
\begin{eqnarray}
    N_{10} + N_{01} < N_{11} + N_{00},
\end{eqnarray}
i.e the number of bad indices which survive is smaller than the number of good indices which survive pruning. Thus, under the previous condition, we must have $\overline C=\overline B>1/2$ and $\overline A=\overline D = 1-\overline B < 1/2$. By symmetry of the preceeding argument we know that the condition is also necessary.
We deduce the following result.
\begin{proposition} 
Suppose $b=-a$. Then, for any solution $(A,B,C,D)$ of the system of equations \eqref{eq:master}, the inequalities
\begin{eqnarray}
    \overline C=\overline B > 1/2,\quad \overline D=\overline A<1/2,
\end{eqnarray}
hold if and only iff $ N_{10} + N_{01} < N_{11} + N_{00}$.

\label{prop:symmetric-case}
\end{proposition}

\paragraph{Skewed Case: $b=0$.}
Here, we have
\begin{eqnarray}
    \begin{split}
    \overline B &= \sigma(a(N_{11}-(N_{11} + N_{10})\overline B)/\gamma),\\
    \overline D &= \sigma(a(N_{01}-(N_{01}+N_{00})\overline D)/\gamma)
\end{split}
\end{eqnarray}
If $\overline B \le 1/2$, then $\overline B \ge N_{11}/(N_{11} + N_{10})$, which is impossible if we impose 
\begin{eqnarray}
    N_{10} < N_{11},
\end{eqnarray}
i.e the number of examples with true label $1$, which are incorrectly labelled as $0$ in the dataset, which survive pruning is less than the number of examples with true label $1$, which are correctly labelled and survive pruning.
We deduce that $\overline B > 1/2$ under the above condition. 

Similarly, if $\overline D \ge 1/2$, then $\overline D \le N_{01}/(N_{01} + N_{00})$, which is impossible if we impose
\begin{eqnarray}
    N_{01} < N_{00},
\end{eqnarray}
i.e the number of  with true label $0$ but incorrectly labelled as $1$ in the dataset, which survive pruning is less than the number of examples with true label $1$, which are correctly labelled and survive pruning. We obtain the following result.
\begin{proposition}
\label{prop:skewed-case}
    Suppose $b=0$. Then, for any solution $(\overline A, \overline B,\overline C,\overline D)$ of \eqref{eq:master}, we have
    \begin{align}
        \overline C,\overline B &> 1/2 \text{ iff }N_{10} < N_{11},\\
        \overline D,\overline A &< 1/2\text{ iff }N_{01} < N_{00}.
    \end{align}
\end{proposition}

\subsection{Concentration}
We shall now derive conditions which are sufficient to ensure the hypothesis in Propositions \ref{prop:symmetric-case} and \ref{prop:skewed-case}, namely $N_{k\ell} < N_{kk}$ for all $k,\ell \in \{0,1\}$ with $k \ne \ell$. 
Recall that for any $k,\ell \in \{0,1\}$, the counter $N_{k\ell}$ is random with binomial distribution $Bin(N,p_{k\ell})$.
Now, by basic binomial concentration, we know that if $p,\psi \in [0,1)$ and $\phi \in (0,1]$, then for any fixed $t \in (0,1)$, it holds w.p $1-o(1)$ that
\begin{align}
     \begin{cases}
     N_{k\ell} \le (1+t)N p_{k\ell},&\mbox{ if }k \ne \ell,\\
     N_{k\ell} \ge (1-t)Np_{k\ell},&\mbox{ if }k = \ell.
     \end{cases}
\end{align}
In particular, w.p $1-o(1)$, it holds that
\begin{align}
    N_{k\ell} &\le (1+t)Np_{k\ell},\\
    N_{kk} &\ge (1-t)Np_{kk}.
\end{align}
Comparing the above inequalities, we deduce the following result.
\begin{proposition}
If the following condition holds
\begin{eqnarray}
   \frac{p_{01} + p_{10}}{p_{00} + p_{11}} < \frac{1-t}{1+t}=1-\epsilon\text{ with }\epsilon := \frac{2t}{1+t},
\end{eqnarray}
then w.p $1-o(1)$ it holds that
\begin{align}
    N_{10} + N_{01} < N_{11} + N_{00}.
\end{align}
\label{prop:concentration}
\end{proposition}

\subsection{Proof of Theorem \ref{thm:main}}
Follows directly from putting together Propositions \ref{prop:acc},  \ref{prop:symmetric-case}, \ref{prop:skewed-case}, and \ref{prop:concentration}, and then solving the inequality
$$
\frac{1-t}{1+t} \le \frac{p_{01}+p_{10}}{p_{00} + p_{11}} = \frac{2p\psi}{2(1-p)\phi} = \frac{p\psi}{(1-p)\phi}
$$
for $p$.
\qed




\end{document}